\documentclass[lettersize,journal]{IEEEtran}
\usepackage{amsmath,amsfonts}
\usepackage{amsthm}
\usepackage{algorithmic}
\usepackage{algorithm}
\usepackage{array}
\usepackage[caption=false,font=normalsize,labelfont=sf,textfont=sf]{subfig}
\usepackage{textcomp}
\usepackage{stfloats}
\usepackage{url}
\usepackage{verbatim}
\usepackage{graphicx}
\usepackage{cite}
\usepackage{booktabs}
\usepackage{thmtools}
\usepackage{thm-restate}
\usepackage{xcolor}
\newtheorem{remark}{Remark}%
\newtheorem{definition}{Definition}

\declaretheorem[name=Lemma]{lemma}

\declaretheorem[name=Proposition]{proposition}

\allowdisplaybreaks[4]

\begin{document}

\title{Preference-Guided Reinforcement Learning for Efficient Exploration}

\author{Guojian~Wang,
        Jianxiang~Liu,
        Xinyuan~Li,
        Faguo~Wu,
        Xiao~Zhang,
        Tianyuan~Chen,
        and~Xuyang~Chen,
\thanks{Manuscript received July 10, 2024.}
\thanks{Xiao Zhang and Faguo Wu are the corresponding authors (e-mail: xiao.zh@buaa.edu.cn, faguo@buaa.edu.cn).}
}

\markboth{Journal of \LaTeX\ Class Files,~Vol.~14, No.~8, August~2021}%
{Shell \MakeLowercase{\textit{et al.}}: A Sample Article Using IEEEtran.cls for IEEE Journals}


\maketitle

\begin{abstract}
In this paper, we investigate preference-based reinforcement learning (PbRL), which enables reinforcement learning (RL) agents to learn from human feedback. This is particularly valuable when defining a fine-grain reward function is not feasible. However, this approach is inefficient and impractical for promoting deep exploration in hard-exploration tasks with long horizons and sparse rewards. To tackle this issue, we introduce LOPE: \textbf{L}earning \textbf{O}nline with trajectory \textbf{P}reference guidanc\textbf{E}, an end-to-end preference-guided RL framework that enhances exploration efficiency in hard-exploration tasks. Our intuition is that LOPE directly adjusts the focus of online exploration by considering human feedback as guidance, thereby avoiding the need to learn a separate reward model from preferences. Specifically, LOPE includes a two-step sequential policy optimization technique consisting of trust-region-based policy improvement and preference guidance steps. We reformulate preference guidance as a trajectory-wise state marginal matching problem that minimizes the maximum mean discrepancy distance between the preferred trajectories and the learned policy. Furthermore, we provide a theoretical analysis to characterize the performance improvement bound and evaluate the effectiveness of the LOPE. When assessed in various challenging hard-exploration environments, LOPE outperforms several state-of-the-art methods in terms of convergence rate and overall performance.
\end{abstract}


\section{Introduction}
\IEEEPARstart{D}{eep} reinforcement learning (DRL) has recently successfully solved challenging problems through trial and error~\cite{mnih2015human,mnih2016asynchronous,silver2016mastering}. However, many real-world tasks involve long horizons and poorly defined goals and specifying a suitable reward function in these tasks is difficult~\cite{Jeon2020RewardrationalC,HadfieldMenell2017InverseRD}. Such hard-exploration tasks remain exceedingly challenging for DRL~\cite{yang2021exploration,guo2020memory}. It is difficult for agents to explore efficiently and obtain highly rewarded trajectories in such environments~\cite{yang2021exploration}. This realistic problem further imposes an extreme urgency for efficient exploration, and overcoming this drawback can considerably expand the possible impact of RL.

An alternative approach to mitigate the sparsity of reward feedback is preference-based reinforcement learning (PbRL)~\cite{wirth2016model, abdelkareem2022advances}. In PbRL, instead of directly receiving the instant reward information on each encountered state-action pair, the agent only obtains 1-bit preference feedback for each state-action pair or trajectory from a human overseer~\cite{Chen2022HumanintheloopPE,kang2023beyond}. To acquire a dense reward function, many prior PbRL methods employ specialized models to learn separate reward functions from human preferences and then train policies using these learned reward functions~\cite{christiano2017deep,lee2021pebble}. Such a method suggests that the agent is instructed to act optimally indirectly. Furthermore, it increases model complexity and in-explainability, creating a potential information bottleneck. 

Most existing PbRL approaches do not consider scaling to hard-exploration tasks~\cite{lee2021pebble,kang2023beyond}. Although learning a separate reward function helps in solving the sparse reward problem, continuously enforcing such preference-based rewards throughout the entire training phase cannot guarantee deep exploration in hard-exploration tasks. Only a few studies encourage exploration based on uncertainty in learned reward functions for PbRL algorithms~\cite{liang2022reward}. However, we have no guarantee of achieving deep exploration with this method in hard-exploration tasks with long horizons and large state spaces. Therefore, scaling PbRL to hard-exploration tasks can fill this research gap in RL.

In this study, we propose an online RL algorithm called \textbf{L}earning \textbf{O}nline with trajectory \textbf{P}reference guidanc\textbf{E} (LOPE). This end-to-end framework jointly encourages exploration and directly learns the optimal control policy by regarding human preferences as guidance. LOPE is designed as an end-to-end framework that directly updates policies based on human preferences and environmental feedback, without requiring an explicit reward model or a separate reward learning phase, as shown in Fig~\ref{fig:lope_framework}. Each training iteration consists of two tightly coupled components: a trust-region-based policy improvement step and a preference-guided exploration mechanism. These components are executed jointly in every iteration and together form a unified policy optimization process. 
Specifically, the first step involves employing a trust-region-based approach to generate a candidate policy based on environmental rewards. In the second step, we introduce a trajectory-based distance between policies based on the maximum mean discrepancy (MMD) and reformulate preference guidance as a trajectory-wise state marginal matching problem. We then show that this trajectory-wise state marginal matching problem can be transformed into a policy-gradient algorithm with shaped rewards learned from human preferences. Furthermore, this approach enables us to provide theoretical performance guarantees for performance improvement while achieving outstanding performance in various benchmark tasks. Extensive experimental results demonstrate the superior performance of LOPE compared to competitive baselines.

Our contributions are summarized as follows: (1) LOPE: a simple, efficient, and end-to-end PbRL approach that addresses the exploration difficulty by regarding human preferences as guidance; (2) a preference-guided trajectory-wise state marginal matching optimization objective and a unified two-step optimization technique in each policy iteration; (3) no requirement for additional neural networks; (4) analytical performance guarantees of LOPE.

The remainder of this paper is organized as follows. Section~\ref{sec:related_work} introduces recent impressive studies of relevant areas. Section~\ref{sec:Preliminaries} briefly describes the important background knowledge. Then, Section~\ref{sec:method} introduces the LOPE approach in detail. Section~\ref{sec:thm} analyzes the performance improvement bound of LOPE. The experimental setups are presented in Section~\ref{sec:setup}. Section~\ref{sec:results} presents the experimental results of LOPE. Finally, Section~\ref{sec:conclusion} summarizes the main conclusion of this study.
\begin{figure}
    \centering
    \includegraphics[scale=0.50]{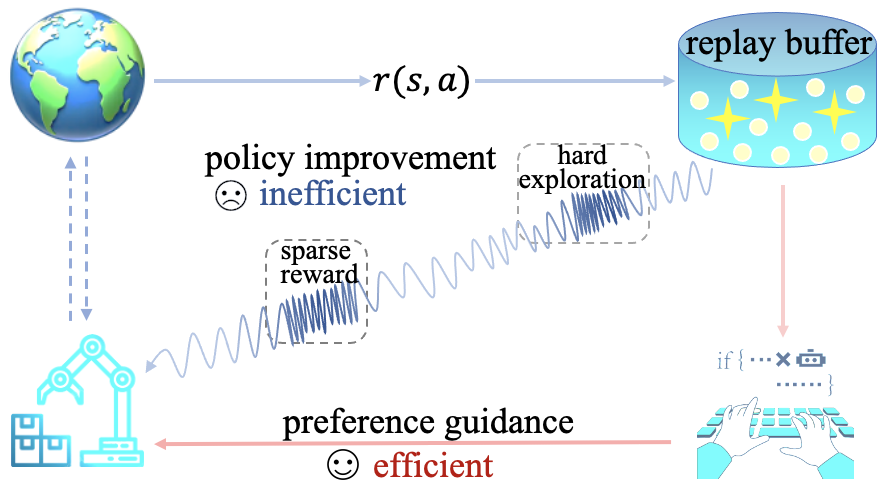}
    \caption{Illustration of our method. Due to the hard exploration and sparse reward, the traditional policy improvement is inefficient. In this work, we propose an end-to-end preference-guided RL framework to achieve efficient exploration.}
    \label{fig:lope_framework}
\end{figure}

\section{Related work}
\label{sec:related_work}
\subsection{Preference-based RL}
Preference-based RL is viewed as human-in-the-loop reinforcement learning, as it learns from human feedback~\cite{wirth2017survey}. Notably, recent studies have proposed using transformer-based models to better capture preference dynamics~\cite{kim2023preference}, learned dynamics models to guide preference-efficient policy learning~\cite{liu2023efficient}, and population-based mechanisms for improved exploration in preference space~\cite{driss2025pb}. Unlike our method, these approaches employ more complex structures to achieve preference modeling and efficient exploration. Overall, there are three main types of preference feedback in the recent PbRL work. The first category of human preference indicates the preferred states between different state pairs, where actions that are better than all available actions exist in the other state~\cite{wirth2017survey}. Second, action preferences are defined on the action space to determine the better action in a given state~\cite{Sugiyama2012PreferencelearningBI,Daniel2015ActiveRL,Asri2016ScorebasedIR,Arumugam2019DeepRL,wirth2016model}. Lastly, trajectory preferences compare the human feedback between trajectory pairs and specify the preferred trajectories over the other ones~\cite{christiano2017deep,Chen2022HumanintheloopPE}. 

Trajectory preferences are the most popular form among the three types of human feedback and are the main focus of many research~\cite{kang2023beyond}. Recently, some impressive work introduced some theoretical results about PbRL with trajectory feedback, which mainly concentrates on the tabular RL setting\cite{Novoseller2019DuelingPS,Xu2020PreferencebasedRL,Pacchiano2021DuelingRR}. ~\cite{kang2023beyond} proposes an end-to-end offline preference-guided RL algorithm to avoid the potential information bottleneck. SURF infers the pseudo-labels of unlabeled samples based on the confidence of the preference predictor for reward learning~\cite{parksurf}. However, these approaches often require learning a separate reward function corresponding to human feedback or an additional model to encode preference information, resulting in a higher computational burden and unstable training. 

\subsection{Learn from Demonstrations (LfD)}
The objective of LfD is to enhance the policy exploration ability of the agent and accelerate learning by combining RL with expert demonstration experience~\cite{schaal1996learning}. Recently, various works have focused on the research field of LfD. \cite{hester2018deep} uses expert demonstrations to design an auxiliary margin loss for a pre-trained deep Q-function. \cite{rajeswaranlearning} shows that model-free DRL algorithms can scale to high-dimensional dexterous manipulation by behavior cloning from a few human expert demonstrations. AWAC accelerates learning of online RL by leveraging a prodigious amount of offline demonstration data with associated rewards~\cite{nair2020awac}. LOGO exploits offline and imperfect demonstrations generated by a sub-optimal policy to achieve faster and more efficient online RL policy optimization~\cite{rengarajanreinforcement}. Self-imitation learning (SIL) methods train the agent to imitate its past self-generated experience only when the return of the trajectory exceeds the value function of a certain threshold~\cite{oh2018self,gangwani2019learning,libardi2021guided}. However, LfD methods place a high demand on high-quality demonstrations, and the performance of these methods is limited by the expert policy that generated the demonstration data.

\subsection{RL with Once-per-episode Feedback}
The goal of RL with once-per-episode feedback is to design a new paradigm that addresses the lack of reward functions in various realistic scenarios where a reward is only revealed at the end of the episode. \cite{efroni2021reinforcement} propose to use an online linear bandit algorithm~\cite{abbasi2011improved} to directly estimate the underlying Markov reward function and introduce a hybrid optimistic-Thompson Sampling method with a $\sqrt{K}$ regret. Meanwhile, \cite{chatterji2021theory} abandons the Markov reward assumption to provide an efficient algorithm based on the recent work~\cite{russac2021self}, obtaining its regret bounds with the UCBVI algorithm~\cite{azar2017minimax}. ~\cite{xu2022provably} proposes the PARTED algorithm that decomposes the trajectory-wise return into a reward sequence for every state-action pair in the trajectory and then performs a pessimistic value function update using the learned rewards. Unlike this, our method utilizes human feedback to guide agents in exploring environments efficiently and can be applied to more complex scenes.

\section{Preliminaries}
\label{sec:Preliminaries}
\subsection{Reinforcement Learning}
We formulate the reinforcement learning problem as a discrete-time Markov decision process (MDP)~\cite{sutton2018reinforcement} described by a tuple $(\mathcal{S}, \mathcal{A}, r, P, \rho_0, \gamma)$. Here, $\mathcal{S}$ and $\mathcal{A}$ denote the spaces of states and actions, respectively. $\gamma \in (0,1)$ is the discount factor. A stochastic policy $\pi(a_t \vert s_t)$ defines an action probability distribution conditioned by $s_t$. At each time step $t$, given an action $a_t$ from $\pi$, the agent samples the next state $s_{t+1}$ from the transition probability distribution $P$, i.e., $s_{t+1} \sim P(s_{t+1}\vert s_t, a_t)$, and it receives a reward $r(s_t, a_t)$ determined by the reward function $r: \mathcal{S}\times\mathcal{A} \rightarrow \mathbb{R}$. $\rho_0$ is the initial state distribution. The objective function is expressed as the expectation of discounted cumulative rewards:
\begin{equation}
  \label{eq:J_return}
  J(\pi) = \mathbb{E}_{\tau\sim\pi}\left[R(\tau)\right].
\end{equation}

Here, $\tau$ denotes a trajectory $\{s_0, a_0, s_1, a_1, \dots\}$, $R(\tau)$ is the sum of the discounted cumulative rewards over $\tau$, $R(\tau)=\sum_{k=0}^\infty \gamma^k r(s_{t+k}, a_{t+k})$. When $\gamma<1$, a discounted state visitation distribution $d_\pi$ can be given as: $d_\pi(s) = (1-\gamma)\sum_{t=0}^{\infty}\gamma^t \mathbb{P}(s_t=s\vert \pi)$, where $\mathbb{P}(s_t=s\vert \pi)$ denotes the probability of $s_t=s$ with respect to the randomness induced by $\pi$, $P$ and $\rho_0$.

\subsection{PbRL with Trajectory Preferences}
In this study, we consider a PbRL framework with trajectory preferences. Unlike the standard RL problem setting, where the agent can receive per-step proxy rewards, we do not assume that such instantaneous rewards are accessible. Instead, the agent can only obtain trajectory preferences revealed by an annotator during the early training process. Given these trajectory preferences, suppose that there is a reward function $r_\psi(s_t, a_t)$ that coincides with these human preferences, the goal of PbRL algorithms is to learn a policy $\pi(a_t\vert s_t)$ to maximize the expected (discounted) cumulative rewards $J_{r_\psi}$. To achieve this goal, vast amounts of previous work have learned a reward function that incorporates these human preferences~\cite{christiano2017deep,ibarz2018reward,stiennon2020learning,wu2021recursively}, which is intended to be consistent with the expert preferences. In contrast, our method directly extracts the policy from these trajectory preferences.

\subsection{State Marginal Matching}
State marginal matching (SMM) has been recently studied as an alternative problem in RL~\cite{lee2019efficient,ghasemipour2020divergence,furuta2021generalized}. Different from the RL objective to maximize the expected return, the SMM objective is to find a policy whose state marginal distribution $\rho_\pi(s)$ minimizes the divergence $D$ to a given target distribution $\rho^\ast(s)$:
\begin{equation}
\mathcal{L}_{\mathrm{SMM}} (\pi) = -D(\rho_\pi(s), p^\ast(s)),
\end{equation}
where $D$ is a divergence measure such as Kullback-Leibler (KL) divergence~\cite{lee2019efficient} and Jensen-Shannon (JS) divergence~\cite{sharma2022state}. For the target distribution $\rho^\ast(s)$, \cite{lee2019efficient} utilizes a uniform distribution to improve exploration in the entire state space; some other approaches leverage imitation learning to match a small set of expert demonstrations~\cite{sharma2022state,ghasemipour2020divergence,kostrikov2019imitation}. 

\subsection{Maximum Mean Discrepancy}
The maximum mean discrepancy (MMD) is an integral probability metric that can be used to measure the difference (or similarity) between two different probability distributions~\cite{gretton2012optimal,masood2019diversity,wang2024trajectory}. Suppose that $p$ and $q$ are the probability distributions defined on a nonempty compact metric space $\mathbb{X}$. Let $x$ and $y$ be observations sampled independently from $p$ and $q$, respectively. Given a Reproducing Kernel Hilbert Space (RKHS) $\mathcal{H}$, the MMD metric is given as:
%
\begin{equation}
  \label{eq:MMD_rkhs}
  \begin{aligned}
    \mathrm{MMD}^2(p, q, \mathcal{H}) = \mathbb{E}[k(x, x^\prime)] - 2\mathbb{E}[k(x, y)] + \mathbb{E}[k(y, y^\prime)],
  \end{aligned}
\end{equation}
where $x, x^\prime\ \mathrm{i.i.d.}\sim p$ and $y, y^\prime\ \mathrm{i.i.d.} \sim q$. Eq.~\eqref{eq:MMD_rkhs} is tractable to estimate the MMD distance between $p$ and $q$ using finite samples because $\mathcal{H}$ only has a kernel function.



\section{Learning Online with Trajectory Preference Guidance}
\label{sec:method}
This section presents a novel online PbRL algorithm called \textbf{L}earning \textbf{O}nline with trajectory \textbf{P}reference guidanc\textbf{E} (LOPE). Our primary objective is to utilize trajectory preferences as guidance for direct policy optimization in hard-exploration tasks, thereby avoiding the potential information bottleneck that arises from learning separate reward functions from these preferences. We first introduce a novel preference-guided trajectory-wise state marginal matching optimization objective by reformulating a constrained optimization problem. Then, we develop a two-step policy optimization framework by integrating a policy improvement step with an additional preference guidance step.

At every point in time, our method maintains a control policy $\pi: \mathcal{S}\rightarrow\mathcal{A}$ parameterized by deep neural networks, and this policy is updated by the following three processes:
\begin{enumerate}
  \item The agent interacts with the environment by the policy $\pi$ and collects a set $\mathcal{B}$ of trajectories $\{\tau^1, \dots, \tau^K\}$. 
  \item Policy improvement (PI): We calculate the policy gradient based on environmental rewards and update policy parameters to increase the expected return.
  \item Preference guidance (PG): We compute the preference-guided gradient and update the current policy toward human preferences.
  \item The agent sends the set of trajectories $\mathcal{B}$ to the annotator for a comparison with trajectories in the preferred-trajectory set $\mathcal{P}$ and updates $\mathcal{P}$.
\end{enumerate}

\subsection{Human-in-loop: Trajectory-Wise State Marginal Matching}
\label{subsec:traj-preference-guided}
Many previous works on PbRL have explicitly learned reward functions that coincide with human preferences. Alternatively, we aim to perform direct policy optimization using a preference-labeled dataset. Inspired by previous research~\cite{lee2019efficient,ghasemipour2020divergence}, we propose reforming preference guidance as a trajectory-wise state marginal matching problem. Specifically, we define the divergence $d(\tau, \upsilon)$ between state visitation distributions of trajectories $\tau$ and $\upsilon$. Then, based on $d(\cdot, \cdot)$, we obtain the following trajectory-wise state marginal matching objective between $\pi$ and $\pi_b$:
%
\begin{equation}
  \label{eq:mmd_objective}
  \begin{aligned}
  \mathcal{L}_{\mathrm{TW}}(\theta)= \underset{\tau \sim \rho_\theta}{\mathbb{E}}\left[\underset{\upsilon \sim \mathcal{P}}{\mathbb{E}} \left[d(\tau, \upsilon)\right]\right],
  \end{aligned}
\end{equation}
where $\rho_\theta$ is a distribution induced by $\pi_\theta$ over the trajectory space. 

Compared with previous studies of state marginal matching~\cite{lee2019efficient,ghasemipour2020divergence}, our approach emphasizes trajectory-level alignment. During training, many state-action pairs are concentrated near the starting point, and the agent only visits novel regions on rare occasions. Hence, due to the low data proportion, these novel state-action pairs have little impact on policy optimization when the SMM objective only considers the divergence between different state visitation distributions. Our approach addresses this problem by introducing a trajectory-wise SMM objective, which encourages the agent to engage in deep exploration.

Now the focus of the problem comes to the definition of $d(\cdot,\cdot)$. Inspired by previous research~\cite{masood2019diversity,Wang2023LearningDP}, a trajectory can be considered a deterministic policy, and the distance between them can be defined as the MMD between their state visitation distributions. Therefore, $d(\cdot,\cdot)$ can be expressed as:
\begin{equation}
  \label{equ:MMD_traj}
  \begin{aligned}
  d(\tau, \upsilon) = &\underset{s_\tau, s_\tau^\prime \sim \rho_\tau}{\mathbb{E}} \left[k\left(s_\tau, s_\tau^\prime\right)\right] 
  - 2\underset{\begin{subarray}{c}s_\tau \sim \rho_\tau \\ s_\upsilon \sim \rho_\upsilon\end{subarray}}{\mathbb{E}} \left[k(s_\tau, s_\upsilon)\right] \\
  &+ \underset{s_\upsilon, s_\upsilon^\prime \sim \rho_\upsilon}{\mathbb{E}} \left[k(s_\upsilon, s_\upsilon^\prime)\right] 
  \end{aligned}
\end{equation}
where $\rho_\tau$ and $\rho_\upsilon$ are the state visitation distributions of $\tau$ and $\upsilon$. 
To guarantee the stability of policy optimization, we further adopt the trust-region-based optimization method similar to TRPO~\cite{schulman2015trust}: 

%
\begin{equation}
  \begin{aligned}
    &\pi_{new} = \underset{\pi_\theta}{\operatorname{arg\,min}} \ \mathcal{L}_{\mathrm{TW}}(\theta) \\
    &\text{s.t.}\ D_{\rm KL}(\pi_\theta, \pi_{old}) \le \delta.
  \end{aligned}
  \label{eq:constraint}
\end{equation}

Intuitively, this preference-guided objective function aids learning by orienting the agent's policy toward human feedback. Specifically, our approach finds a new policy $\pi_{new}$ that minimizes the expected trajectory-wise MMD distance to $\pi_b$, while ensuring $\pi_{new}$ lies inside the trust region around $\pi_{old}$. This preference-guided policy optimization is the key distinguishing LOPE from other PbRL approaches. 
This method equips LOPE with two unique advantages over other state-of-the-art algorithms.

First, many previous PbRL studies explicitly learn separate reward functions and then employ off-the-shelf RL algorithms to train a control policy. In contrast, LOPE considers preference-labeled trajectories as guidance and reformulates a novel constrained optimization problem, as shown in Eq.~\eqref{eq:constraint}. In this manner, LOPE effectively avoids the potential information bottleneck that occurs when conveying information from preferences to the policy via scalar rewards. 

Second, LOPE learns a policy matching the trajectory-wise state visitation distribution of preference-labeled experiences. Consequently, LOPE enables the agent to quickly expand the scope of exploration along the preference-labeled trajectories in $\mathcal{P}$ and overcome the challenging exploration difficulty in online RL settings. Fundamentally, the proposed approach enables us to explore the available preference data more comprehensively, including position information, thereby providing more significant assistance to RL algorithms.

\subsection{Two-step Policy Optimization Framework}
This section presents a two-step policy optimization framework that provides exploration aid and accelerates learning for the agent. Formally, the policy $\pi_\theta$ is updated by the following process:

\textbf{Step 1: Policy Improvement} 

LOPE performs one-step policy improvement by optimizing the standard trust-region-based objective function~\cite{schulman2015trust}. This can be expressed as follows: 
\begin{equation}
  \label{eq:policy_improvement}
  \begin{aligned}
  &\pi_{k+\frac{1}{2}} = \underset{\pi_\theta}{\operatorname{arg\,max}} \ \underset{\begin{subarray}{l} s \sim d_k \\ a \sim \pi_\theta \end{subarray}}{\mathbb{E}} \left[A_{k}(s,a)\right], \\
    &\text{s.t.}\ D_{\rm KL}(\pi_\theta, \pi_{k}) \le \delta.
  \end{aligned}
\end{equation}

Here, $d_k$ is the discounted state visitation distribution corresponding to $\pi_k$, and $A_{k}(s, a)$ is the advantage of $\pi_k$ in the state-action pair $(s, a)$. The trust-region-based update in Eq.~\eqref{eq:policy_improvement} is to find a new policy $\pi_{k+\frac{1}{2}}$ that maximizes the RL objective while ensuring this new policy stays within the trust region around $\pi_k$. From an implementation perspective, off-the-shelf algorithms can solve this optimization problem, such as TRPO~\cite{schulman2015trust} or PPO~\cite{schulman2017proximal}. 

\textbf{Step 2: Preference Guidance}

Many RL approaches achieve efficient online policy optimization in dense reward settings. Nonetheless, they may often fail due to the sparsity of reward functions. It is a forlorn hope to achieve significant performance improvement with sparse-reward feedback in the initial training stage. We propose overcoming this difficulty in exploration by using preference as guidance and adjusting policy optimization directly. The preference guidance step is given as follows:
\begin{equation}
  \label{eq:perference_guidance}
  \begin{aligned}
    &\pi_{k+1} = \underset{\pi_\theta}{\operatorname{arg\,min}} \ \underset{\tau \sim \rho_\theta}{\mathbb{E}}\left[\underset{\upsilon \sim \mathcal{P}}{\mathbb{E}} \left[d(\tau, \upsilon)\right]\right], \\
    &\text{s.t.}\ D_{\rm KL}(\pi_\theta, \pi_{k+\frac{1}{2}}) \le \delta. 
  \end{aligned}
\end{equation}

This preference guidance step is derived from Eq.~\eqref{eq:constraint}. We solve Eq.~\eqref{eq:perference_guidance} by introducing a novel policy gradient. Instead of directly computing the gradient update direction to maximize Eq.~\eqref{eq:perference_guidance} as~\cite{masood2019diversity}, we seek to solve an equivalent optimization problem, which is obtained by expanding the objective function in Eq.~\eqref{eq:perference_guidance}. This divergence minimization problem can be reduced to a policy-gradient algorithm with shaped rewards computed from $\mathcal{P}$.

Noting that the equation of MMD in Eq.~\eqref{equ:MMD_traj} can be rewritten as:
\begin{equation}
\label{eq:MMD_traj_two}
  \begin{aligned}
  d(\tau, \upsilon) = &\underset{s_\tau \sim \rho_\tau}{\mathbb{E}}\Big[\underset{s_\tau^\prime \sim \rho_\tau}{\mathbb{E}} \left[k\left(s_\tau, s_\tau^\prime\right)\right] - 2\underset{s_\upsilon \sim \rho_\upsilon}{\mathbb{E}} \left[k(s_\tau, s_\upsilon)\right] \\
  &+ \underset{s_\upsilon, s_\upsilon^\prime \sim \rho_\upsilon}{\mathbb{E}} \left[k(s_\upsilon, s_\upsilon^\prime)\right]\Big].
  \end{aligned}
\end{equation}
Hence, the definition of $d(\cdot, \cdot)$ can be regarded as taking expectations of the distance from the state $s_\tau$ to the trajectory $\upsilon$:
\begin{equation}
\label{eq:MMD_traj_point}
  \begin{aligned}
  d(\tau, \upsilon) = &\underset{s_\tau \sim \rho_\tau}{\mathbb{E}}\left[\operatorname{dist}(s_\tau, \upsilon)\right],
  \end{aligned}
\end{equation}
where $\operatorname{dist}(s_\tau, \upsilon)$ is computed as:
\begin{equation}
\label{eq:dist}
\begin{aligned}
\operatorname{dist}(s_\tau, \upsilon) = &\underset{s_\tau^\prime \sim \rho_\tau}{\mathbb{E}} \left[k\left(s_\tau, s_\tau^\prime\right)\right] - 2\underset{s_\upsilon \sim \rho_\upsilon}{\mathbb{E}} \left[k(s_\tau, s_\upsilon)\right]\\
&+ \underset{s_\upsilon, s_\upsilon^\prime \sim \rho_\upsilon}{\mathbb{E}} \left[k(s_\upsilon, s_\upsilon^\prime)\right]
\end{aligned}
\end{equation}
\begin{restatable}[Preference guidance]{lemma}{mainlem}
\label{lem:pref_guide}
  Let $r_g(s, a)$ denote the preference guidance-based rewards, and it is expressed as:
  \begin{equation}
    r_g(s, a)=\underset{\begin{subarray}{c} \upsilon\in\mathcal{P}\end{subarray}}{\mathbb{E}} \left[\mathrm{dist}(s, \upsilon)\right].
  \end{equation}
  $\operatorname{dist}(\cdot, \cdot)$ is defined in Eq.~\eqref{eq:dist}. Then, Eq.~\eqref{eq:perference_guidance} can be expanded as follows:
  \begin{equation}
    \label{eq:expanded_perference_guidance}
    \begin{aligned}
      &\pi_{k+1} = \underset{\pi_\theta}{\operatorname{arg\,min}} \ \underset{\begin{subarray}{l} s \sim d_\theta \\ a \sim \pi_\theta \end{subarray}}{\mathbb{E}} \left[r_g(s,a)\right], \\
      &\text{s.t.}\ D_{\rm KL}(\pi_\theta, \pi_{k+\frac{1}{2}}) \le \delta.
    \end{aligned}
  \end{equation}
  Here, $d_\theta$ is the discounted state visitation distribution defined in Section~\ref{sec:Preliminaries}
\end{restatable}

We provide the proof in Appendix~\ref{sec:proof_lem1}. Compared to offline demonstrations as guidance~\cite{rengarajanreinforcement}, our approach considers the human-in-the-loop setting where the agent obtains preference-labeled trajectories by interacting with the expert annotator. Therefore, although sparse rewards are unavailable to the agent during the initial training phase, LOPE can guide policy optimization by capturing more task-related information from human feedback. On the other hand, if the support sets of $\pi$ and $\pi_b$ are far apart or do not overlap at all, then the value of the KL divergence is meaningless. In this case, the method of~\cite{rengarajanreinforcement} does not work, but we provide a feasible approach to overcome the gradient vanishing problem. The entire procedure of LOPE is summarized in Algorithm~\ref{algo:lope}. Hence, compared with the result of Lemma 1 in the literature~\cite{rengarajanreinforcement}, our result of Lemma~\ref{lem:assumption-improvement} has better generalizability because of the continuity of the MMD distance.

To solve the optimization problem in Eq.~\eqref{eq:expanded_perference_guidance}, all that remains is to replace the expectation with sample averages, and we use the preference guidance rewards to estimate the expected returns. In this paper, we implement LOPE based on PPO. PPO utilizes a clipped surrogate objective to implicitly constrain policy updates within a trust region, thereby eliminating the need to explicitly set or tune $\delta$. The entire procedure of LOPE is summarized in Algorithm~\ref{algo:lope}.
\begin{remark}
  \label{re:n-wise-comparisons}
  \textbf{Update Preferred Trajectory Sets with $n$-wise Comparisons.} Various previous work adopts the setting where the agent can only access expert preferences $y$ for segments $\sigma^0$ and $\sigma^1$. Specifically, $y\in\{0, 0.5, 1\}$ is sampled from a discrete distribution to indicate which segment the human annotator prefers. Then, the agent records the judgment as a triple $(\sigma^0, \sigma^1, y)$ in a data set $\mathcal{D}$. However, this setting cannot cover our PbRL situation with trajectory preferences. 
  
  LOPE maintains a buffer of preferred trajectories and updates this buffer with current trajectories in each iteration. Therefore, the human annotator must compare all current trajectories with the preferred trajectories in $\mathcal{P}$ in pairs. The trajectories are ranked by human feedback, and the best $h$ trajectories (segments) can be stored in $\mathcal{P}$. In practice, we set $h=8$ in all experiments. To achieve the $n$-wise comparisons, each current trajectory should be compared pair-wise with all preference trajectories in $\mathcal{P}$. When the agent samples $N$ new trajectories in each epoch, the computational complexity of $n$-wise comparisons is $Nh$. As the training progresses, the number of trajectories used to update $\mathcal{P}$ decreases, thereby reducing the experts' workload. The decreased number is environmentally dependent and relies on the agent's training progress. Furthermore, our method's experts only need to assess the overall state visitation situation of the agent's trajectory, without requiring a detailed understanding of the agent's dynamics. This reduces the requirements for experts.
\end{remark}
\begin{algorithm}[ht]
  \caption{LOPE}\label{algo:lope}
  \begin{algorithmic}[1]
    \STATE Initialize parameters of $\pi_\theta$
    \STATE Initialize a data set of trajectory preferences $\mathcal{P}$
    \FOR{each iteration}
    \STATE{{\textsc{// Collect trajectories}}}
    \FOR{$k=1,\dots, K$}
    \STATE Generate batch of $N$ trajectories $\{\tau_i\}_{i}^{n}$ and store them in an on-policy buffer $\mathcal{B}$
    \ENDFOR
    \STATE{{\textsc{// Update policies}}}
    \FOR{each gradient step}
    \STATE Optimize the trust-region-based objective in Eq.~\eqref{eq:policy_improvement} with respect to $\theta$ using $\mathcal{B}$
    \STATE Optimize the preference-guided objective in Eq.~\eqref{eq:perference_guidance} using $\mathcal{B}$ and $\mathcal{P}$
    \ENDFOR
    \STATE{{\textsc{// Update preferred trajectory sets}}}
    \FOR{each trajectory in $\mathcal{B}$}
    \STATE Update preferred trajectory set $\mathcal{P}$ with $n$-wise comparisons
    \ENDFOR
    \ENDFOR
  \end{algorithmic}
\end{algorithm}

\section{Theoretical Analysis}
\label{sec:thm}
This section analyzes the policy improvement step and preference guidance step separately. We make the following assumption about the behavior policy $\pi_b$ implied by the preferred trajectories in set $\mathcal{P}$. 
\begin{figure}[ht]
    \centering
    \includegraphics[width=0.8\linewidth]{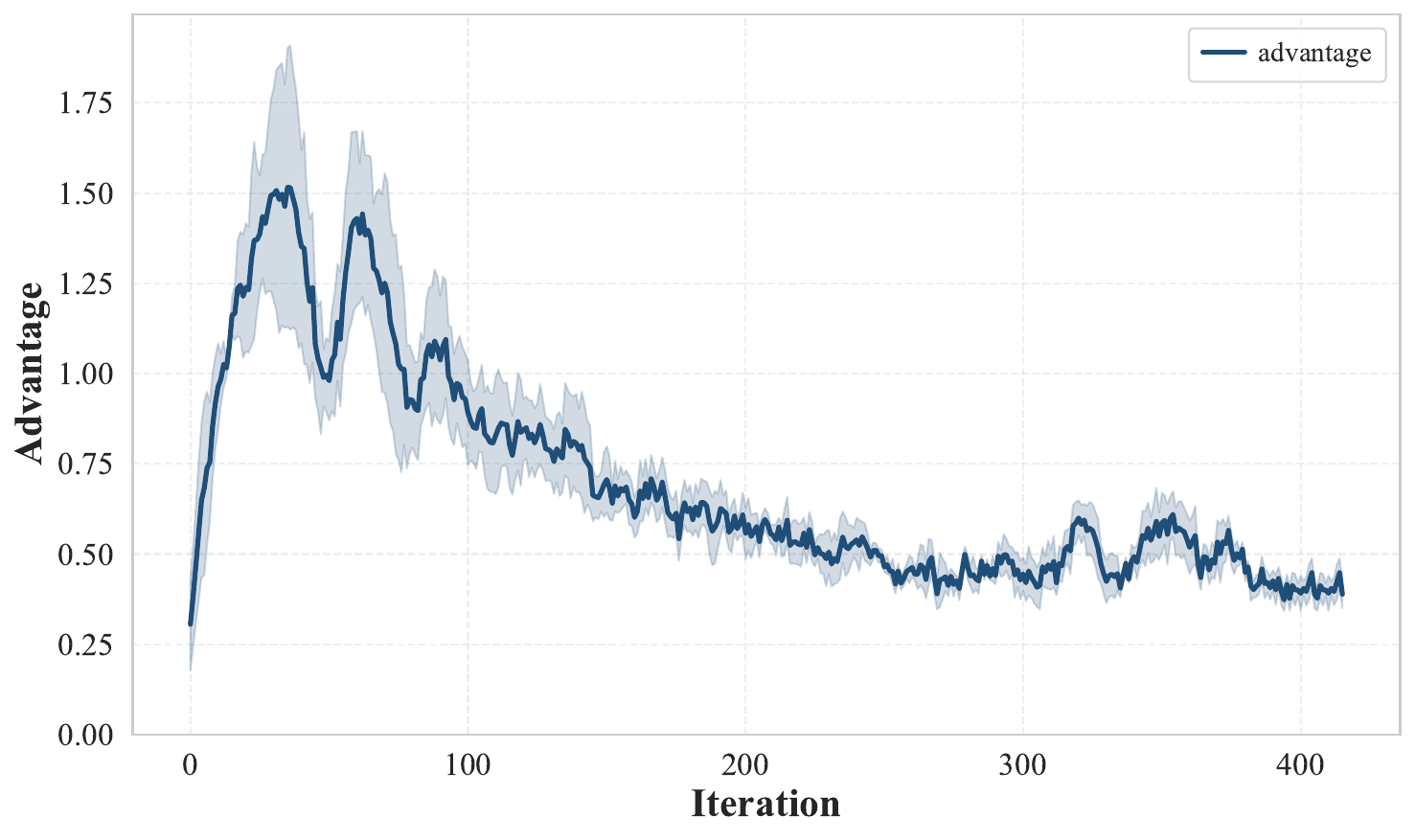}
    \caption{The advantage of the guiding policy $\pi_b$ over the current policy.}
    \label{fig:adv}
\end{figure}

%
%
\begin{restatable}{assumption}{firstass}
  \label{ass:1}
  $\pi_b$ is a \textbf{guiding} policy, and for any iteration $k$ during training, it satisfies the following conditions:
  %
  \begin{equation}
    \label{eq:ass1_A}
    \mathbb{E}_{a\sim\pi_b}\left[A_{k}(s,a)\right]\ge\Delta>0,
  \end{equation}
  where $\rho_k$ is the discounted state visitation distribution of $\pi_k$. 
\end{restatable}

To empirically validate Assumption~\ref{ass:1}, we evaluate the guided advantage $\Delta=\mathbb{E}_{a\sim\pi_b}\left[A_{k}(s,a)\right]$ in the grid-world maze task with randomized goals. As shown in Fig.~\ref{fig:adv}, the advantage of $\pi_b$ over $\pi_k$ remains consistently positive and above a threshold ($\approx 0.15$), indicating that the guiding policy reliably satisfies the sufficient condition assumed in our theoretical analysis.

\begin{figure*}[htb]
  \centering
  \subfloat[]{
  \includegraphics[width=4.3cm]{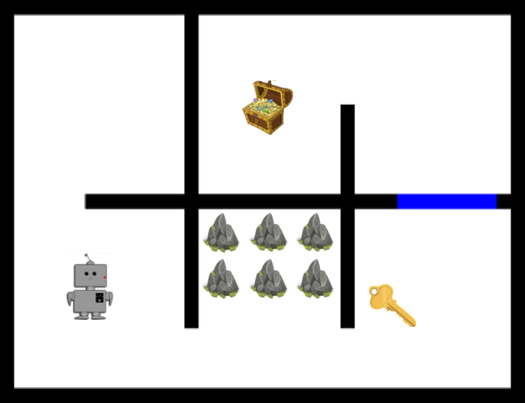}
  \label{fig:grid_maze}
  }
  \subfloat[]{
  \includegraphics[width=3.32cm]{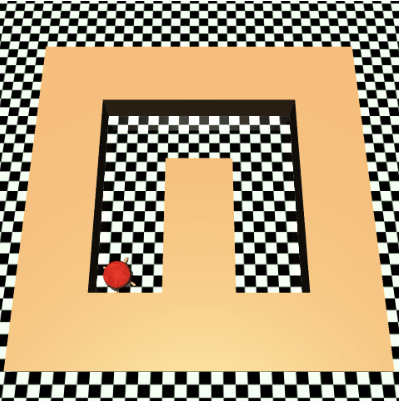}
  \label{fig:ant_maze}
  }
  \subfloat[]{
  \includegraphics[width=3.45cm]{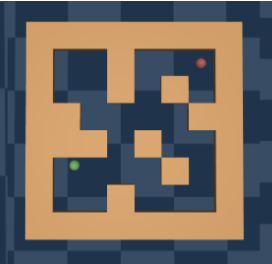}
  \label{fig:maze_medium}
  }
  \subfloat[]{
  \includegraphics[width=3.35cm]{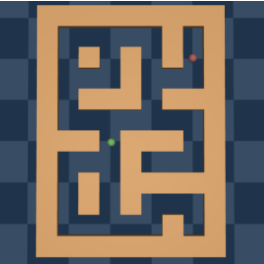}
  \label{fig:maze_large}
  }
  \caption{The four environments for evaluating the LOPE's performance: (a) Grid World; (b) AntMaze-Umaze; (c) PointMaze-Medium; (d) PointMaze-Large.}
\end{figure*}

Eq.~\eqref{eq:ass1_A} implies that the agent will produce a higher advantage by performing actions according to the demonstrations in $\mathcal{P}$. This is a relatively reasonable assumption since the demonstration preferred by the annotator is likely to achieve better performance than a policy that is not fully trained. It is worth mentioning that a similar assumption is adopted by~\cite{rengarajanreinforcement} and~\cite{kang2018policy}. However, our method does not require an exact estimation of $\rho_\pi$ during implementation. Instead, we rely on sampling-based approximations using on-policy trajectories. $\pi_b$ serves as a heuristic guide, providing direction for exploration based on human preferences. The expected return of $\pi_b$ may gradually increase as the agent stores more preferred trajectories with high returns.

We begin with the analysis of the policy improvement step, and the following conclusion and its analysis are standard in the related literature~\cite{achiam2017constrained}.
\begin{proposition}[Performance improvement bound of the PI step~\cite{achiam2017constrained}]
  \label{prop:policy_improve}
  Let any policies $\pi_k$ and $\pi_{k+\frac{1}{2}}$ are related by Eq.~\eqref{eq:policy_improvement}, Then, the following bound holds:
  \begin{align}
    \label{eq:trpo-guarantee}
    J(\pi_{k+\frac{1}{2}}) - J(\pi_{k}) \ge \frac{- \sqrt{2 \delta} \gamma \epsilon_{k} }{(1 - \gamma)^{2}},
  \end{align}
  where $\epsilon_{k} = \max_{s, a} |A_{k}(s, a)|$.
\end{proposition}



We now give the following lemma to connect the difference in returns between $\tilde{\pi}$ and $\pi$ to the KL divergence between $\tilde{\pi}$ and $\pi_b$.
\begin{restatable}[Performance improvement bound of the PG step]{lemma}{preferimprove}
  \label{lem:assumption-improvement}
  For arbitrary policies $\pi$ and $\tilde{\pi}$, let $\beta=D_{\rm KL}^{\max}(\tilde{\pi}, \pi_b)=\max_s D_{\rm KL}(\tilde{\pi}(\cdot\vert s), \pi_b(\cdot\vert s))$, and $\pi_b$ is a policy satisfying Assumption~\ref{ass:1}. Then we have
  \begin{equation}
    \begin{aligned}
      \label{eq:assumption-improvement}
     J(\tilde{\pi}) - J(\pi) \ge -\frac{2\beta\gamma\epsilon_b}{(1-\gamma)^2} +\frac{\Delta}{1-\gamma},
    \end{aligned}
  \end{equation}
  where $\epsilon_{b} = \max_{s,a}|A_{b}(s,a)|$ is the maximum absolute value of $\pi_b$'s advantage function $A_b(s,a)$. 
\end{restatable}

The proof can be found in Appendix~\ref{sec:proof_bound}. 

\begin{remark}
  Lemma~\ref{lem:assumption-improvement} provides a quantitative description of performance improvement in the preference-guided step. The human annotator selects preferred trajectories and stores them in $\mathcal{P}$ during the training process. Hence, it is reasonable to assume that $\pi_{k+\frac{1}{2}}$ in the PG step satisfies Assumption~\ref{ass:1} in the preference-guided step. Based on these preferred trajectories, the agent obtains a higher worst-case lower bound of performance improvement. Specifically, according to Lemma~\ref{lem:assumption-improvement}, for $\pi=\pi_{k+\frac{1}{2}}$, Eq.~\eqref{eq:assumption-improvement} indicates that the new lower bound increases by $\Delta/(1-\gamma)$ than in cases where preferred trajectories are not provided. 
\end{remark}

We derive a new worst-case lower bound of LOPE's performance improvement by combining Proposition~\ref{prop:policy_improve} and Lemma~\ref{lem:assumption-improvement}. The following theorem presents an explicit per-iteration lower bound on the policy improvement achieved by LOPE. This provides a concrete and interpretable expression characterizing the extent of policy improvement as a function of preference advantage, KL divergence, and reward estimation error.

\begin{restatable}[Performance improvement bound of LOPE]{theorem}{mainthm}
  \label{thm:bound}
  Let $\pi_k$ and $\pi_{k+\frac{1}{2}}$ be related by Eq.~\eqref{eq:policy_improvement} and $\pi_{k+\frac{1}{2}}$ and $\pi_{k+1}$ be related by Eq.~\eqref{eq:perference_guidance}, respectively. Let $\beta_{k+1}=D_{\rm KL}^{\max}(\pi_{k+1}, \pi_b)=\max_s D_{\rm KL}(\pi_{k+1}(\cdot\vert s), \pi_b(\cdot\vert s))$. If $\pi_{k+\frac{1}{2}}$ and $\pi_{k+1}$ satisfies Assumption~\ref{ass:1}, then the following bound holds:
  \begin{equation}
    \begin{aligned}
      \label{eq:thm-g1}
      J(\pi_{k+1}) - J(\pi_{k}) \geq \frac{\Delta}{1-\gamma} -\frac{2\beta_{k+1}\gamma\epsilon_b}{(1-\gamma)^2}  - \frac{\sqrt{2 \delta} \gamma \epsilon_{k}}{(1 - \gamma)^{2}}.
    \end{aligned}
  \end{equation}
  Here, $\epsilon_{k}$ and $\epsilon_{b}$ are as defined in Proposition~\ref{prop:policy_improve} and Lemma~\ref{lem:assumption-improvement}, respectively.
\end{restatable}

We provide the proof in Appendix~\ref{sec:proof_bound}. This bound implies that, when the guided advantage $\Delta$ exceeds the regularization terms, LOPE guarantees monotonic policy improvement. Although the theoretical bound includes a strictly positive term, this quantity serves as a conservative regularization to prevent overly aggressive updates. Our empirical findings indicate that LOPE can still converge to near-optimal solutions in practice.

Specifically, when the behavior policy $\pi_b$ performs much better than the current policy as described in Assumption~\ref{ass:1}, compared to the original worst-case lower bound obtained by the TRPO algorithm, the preference-guided step results in an additional performance improvement of $\Delta/(1-\gamma)$ size. Preference-labeled trajectories offer a more effective and accurate policy optimization direction for the agent, while reducing the sampling complexity required for learning. Thus, LOPE can achieve faster learning.

\begin{remark}[Comparison with Prior Preference-Based Bounds]
While most existing preference-based RL methods focus on regret-based frameworks~\cite{Pacchiano2021DuelingRR,Chen2022HumanintheloopPE} or reward estimation convergence~\cite{biyik2022learning}, they typically do not provide formal guarantees on policy improvement.
In contrast, LOPE introduces a two-step optimization scheme consisting of a preference-guided advantage improvement objective and a KL-regularized policy update, which together yield the following worst-case lower bound:
\[
J(\pi_{k+1}) - J(\pi_k) \ge \frac{\Delta}{1 - \gamma} - \mathcal{O}\left(\frac{\beta A_{\max} + \sqrt{\delta} A_{\max}}{(1 - \gamma)^2} \right),
\]
This bound characterizes the explicit impact of preference quality, policy divergence, and reward variance on performance improvement.

To the best of our knowledge, LOPE is among the first PbRL methods to provide such a per-iteration improvement guarantee without requiring a learned reward function. While similar theoretical guarantees have been studied in the learning-from-demonstration (LfD) literature~\cite{rengarajanreinforcement,kang2018policy}, they are derived under fundamentally different assumptions, typically requiring access to expert trajectories, cost shaping, or supervised behavior cloning. In contrast, LOPE operates purely at the trajectory level, making our theoretical result novel and better aligned with real-world feedback modalities. These contributions help bridge a theoretical gap in preference-based reinforcement learning.
\end{remark}




\section{Experimental Setups}
\label{sec:setup}
We evaluate LOPE on an enormous grid-world task and several continuous control tasks based on the MuJoCo physical engine~\cite{todorov2012mujoco}. The hyperparameters used by the PyTorch code are introduced in Appendix~\ref{sec:nn_parameters}. In our experiments, sparse rewards are generated by the environment, and the agent has access to both sparse rewards and human feedback.

\subsection{Environments}
\textbf{Key-Door-Treasure domain.} As shown in Figure~\ref{fig:grid_maze}, the Key-Door-Treasure domain is a grid-world environment similar to that defined by~\cite{guo2020memory} with minor modifications. The agent should pick up the key (K) to open the door (D) and finally collect the treasure (T). The agent can only obtain a reward of 400 when reaching the treasure, and it cannot receive any rewards in other cases. The size of this grid-world environment is $26\times 36$. The paper uses mazes with fixed and stochastic goals to test the performance of LOPE. The maze environment with stochastic goals selects a random goal in the upper-right or medium rooms at the beginning of each episode.

\textbf{Continuous maze tasks.} To further demonstrate the feasibility of LOPE, we evaluated LOPE on the continuous maze tasks: PointMaze-Medium, PointMaze-Large, and AntMaze-Umaze task, which are shown in Figs~\ref{fig:ant_maze}-\ref{fig:maze_large}. These tasks are designed as a benchmark for RL by~\cite{duan2016benchmarking}. Hard exploration manifests two aspects: Locomotion and Maze. More specifically, the agent must first learn to walk and then reach the target point. The agent is only rewarded for reaching the specified position in each maze. The observation space of this task is naturally decomposed into two parts: the agent's joint angle or position information and task-specific attributes. The task-specific attributes include sensor readings and the positions of walls and goals.
\begin{figure}[htb]
  \centering
  \subfloat[]{
    \label{fig:cheetah}
    \includegraphics[width=3.7cm]{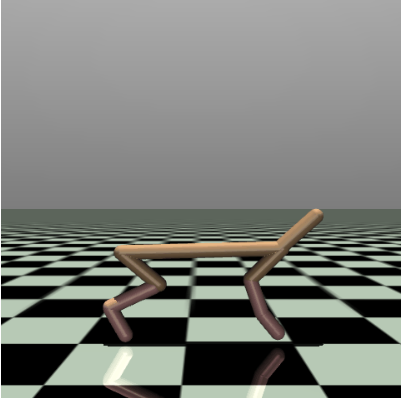}
    }
  \subfloat[]{
    \label{fig:hopper}
    \includegraphics[width=3.65cm]{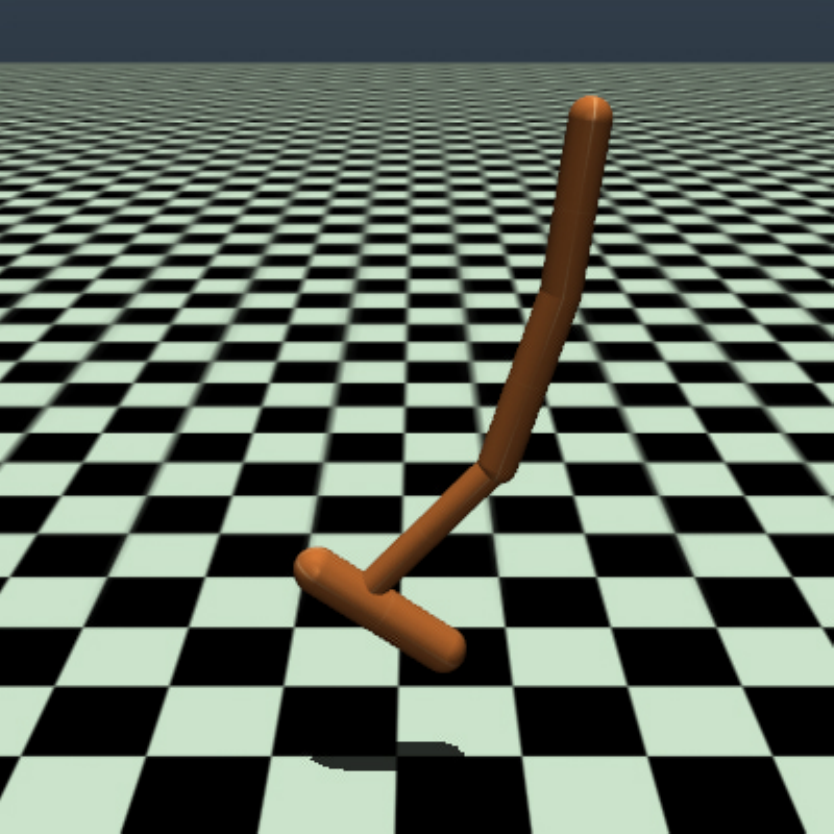}
    }
  \caption{The two locomotion control tasks: (a) SparseHalfCheetah; (b) SparseHopper.}
\end{figure}

\textbf{Locomotion control tasks.} To explore the limitation of LOPE, we modified 2 MuJoCo tasks, HalfCheetah and Hopper, and obtained two new robots: SparseHalfCheetah and SparseHopper presented in Figs.~\ref{fig:cheetah} and~\ref{fig:hopper}. Each of the new robots yielded a sparse reward only when the center of mass of the corresponding robot was beyond a certain threshold distance from its initial position, and the agent could not be rewarded in any other case. 

%

\subsection{Baseline Methods}
For evaluation, we compare LOPE with PEBBLE~\cite{lee2021pebble} and RUNE~\cite{liang2022reward}, the current state-of-the-art PbRL methods. Additionally, our method aims to overcome exploration difficulties; therefore, NoisyNet~\cite{fortunato2017noisy} and PPO+EXP~\cite{tang2017exploration} are adopted as baseline methods. We adopt SIL~\cite{oh2018self} and GASIL~\cite{guo2018generative} as two baselines to demonstrate the LOPE's effectiveness in exploiting previous good experiences. Meanwhile, SMM~\cite{lee2019efficient} solves the exploration problem by proposing a policy-level state marginal matching objective. We use this method as a baseline to demonstrate the effectiveness of the preference-based trajectory-level state marginal matching objective. Lastly, we use the standard PPO~\cite{schulman2017proximal} as a baseline method. 

PEBBLE uses an entropy-based sampling method to improve sample and exploration efficiency. RUNE is a PbRL-based exploration method that defines an intrinsic reward by computing the difference between learned reward models. For fairness, PEBBLE and RUNE adopt the setting where the agent can only access expert preferences $y$ for complete trajectories $\tau^0$ and $\tau^1$. Specifically, $y\in\{0, 0.5, 1\}$ is sampled from a discrete distribution to indicate which segment the human annotator prefers. Then, the agent records the judgment as a triple $(\tau^0, \tau^1, y)$ in a dataset $\mathcal{D}$. Each time the dataset is updated, the reward model is trained to remain consistent with the preference data.

Compared with PEBBLE and RUNE, the primary differences in our approach are (1) an end-to-end PbRL formulation, (2) the subtle combination of online optimization and preference guidance, and (3) the preferred set update method with $n$-wise comparison.


\section{Experimental Results}
\label{sec:results}
This section compares LOPE's performance with several state-of-the-art baseline methods on extensive benchmark control tasks. These average returns were calculated over ten separate runs with different random seeds, and the shaded error bars represented the standard errors. We evaluate the agent's performance quantitatively by measuring the average return, computed using ground truth rewards from environments, for fairness.

\subsection{Key-Door-Treasure Domain}
In the Key-Door-Treasure domain, we designate the positions of the key, door, and treasure room entrance as nodes. Trajectories containing events such as picking the key, opening the door, or entering the treasure room are considered preference trajectories. Moreover, trajectories closer to the goal with fewer steps are preferred. LOPE only maintains the top-$h$ trajectories and leverages them to guide policy optimization; we set $h=5$ in all experiments.
\begin{figure}[htb]
  \centering
  \subfloat[]{
    \label{fig:grid_rate}
    \includegraphics[width=4.2cm]{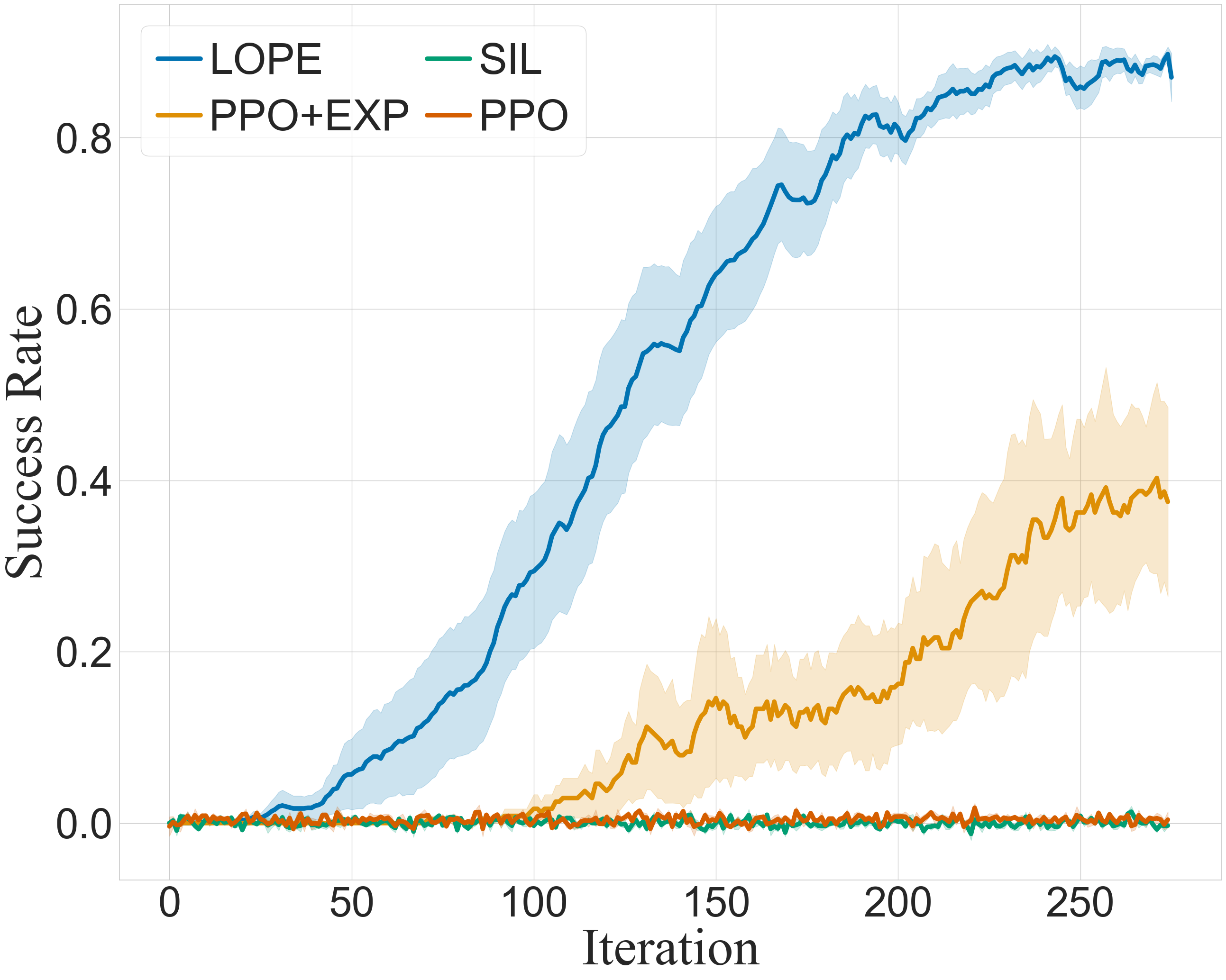}
    }
  \subfloat[]{
    \label{fig:grid_rate_random}
    \includegraphics[width=4.2cm]{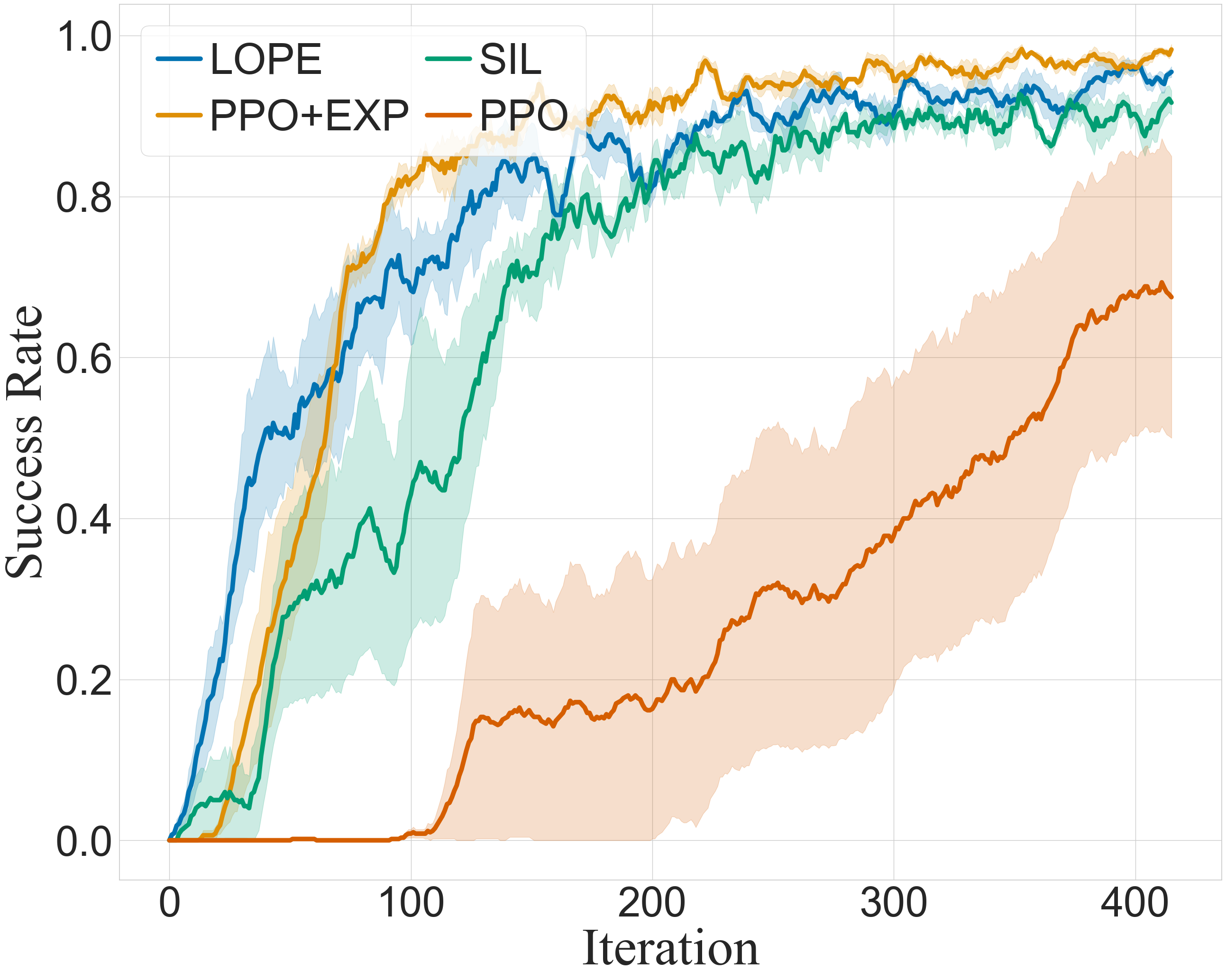}
    }
  \caption{(a) Learning curves of success rate in the grid-world maze with a fixed goal; (b) Learning curves of success rate in the grid-world maze with random goals.}
\end{figure}

As shown in Figs.~\ref{fig:grid_rate} and~\ref{fig:grid_rate_random}, in both grid-world mazes with fixed and random goals, LOPE's performance is superior or equal to other baseline methods. Specifically, LOPE can learn faster and achieve a success rate of nearly 95\%. We noted that in the grid-world maze with the fixed goal, the baseline methods, such as PPO and SIL, could not even learn a sub-optimal policy that picks up the key and then opens the door for a long time. PPO+EXP learned faster than PPO and SIL because this algorithm could explore the environment more efficiently and find the treasure. Interestingly, all baseline methods obtain better performance in the random grid-world environment than in the fixed grid-world maze. The reason for this phenomenon is that in the random grid-world maze, the goal generated in the upper-right room can be reached easily. In both grid-world mazes, LOPE exploited preference-labeled experiences and quickly learned to open the door and further explore the broader state space. This increased the chance of obtaining the treasure reward and helped the agent learn the optimal policy during training. 
%
\begin{figure}[htb]
    \centering
    \subfloat[]{
        \label{fig:act_trajs}
        \includegraphics[scale=0.28]{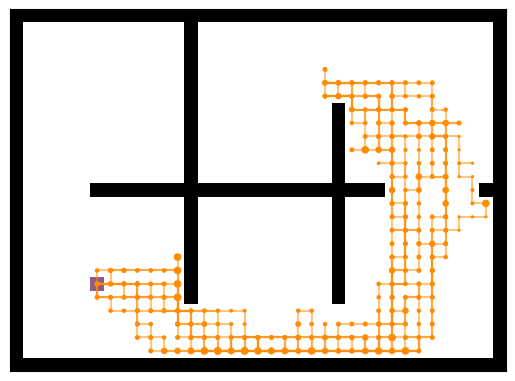}
    }
    \subfloat[]{
        \label{fig:pref_trajs}
        \includegraphics[scale=0.28]{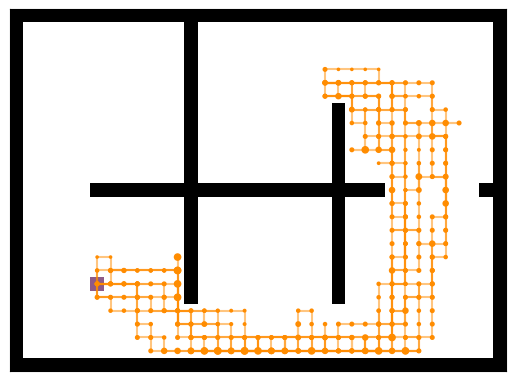}
    }
    \caption{(a) Actual trajectories of the learned optimal policy, (b) Preference-labeled trajectories selected by the annotator.}
    \label{fig:grid}
\end{figure}

\begin{figure*}[htb]
    \centering
    \subfloat[]{
      \label{fig:ant_rate}
      \includegraphics[width=4.5cm]{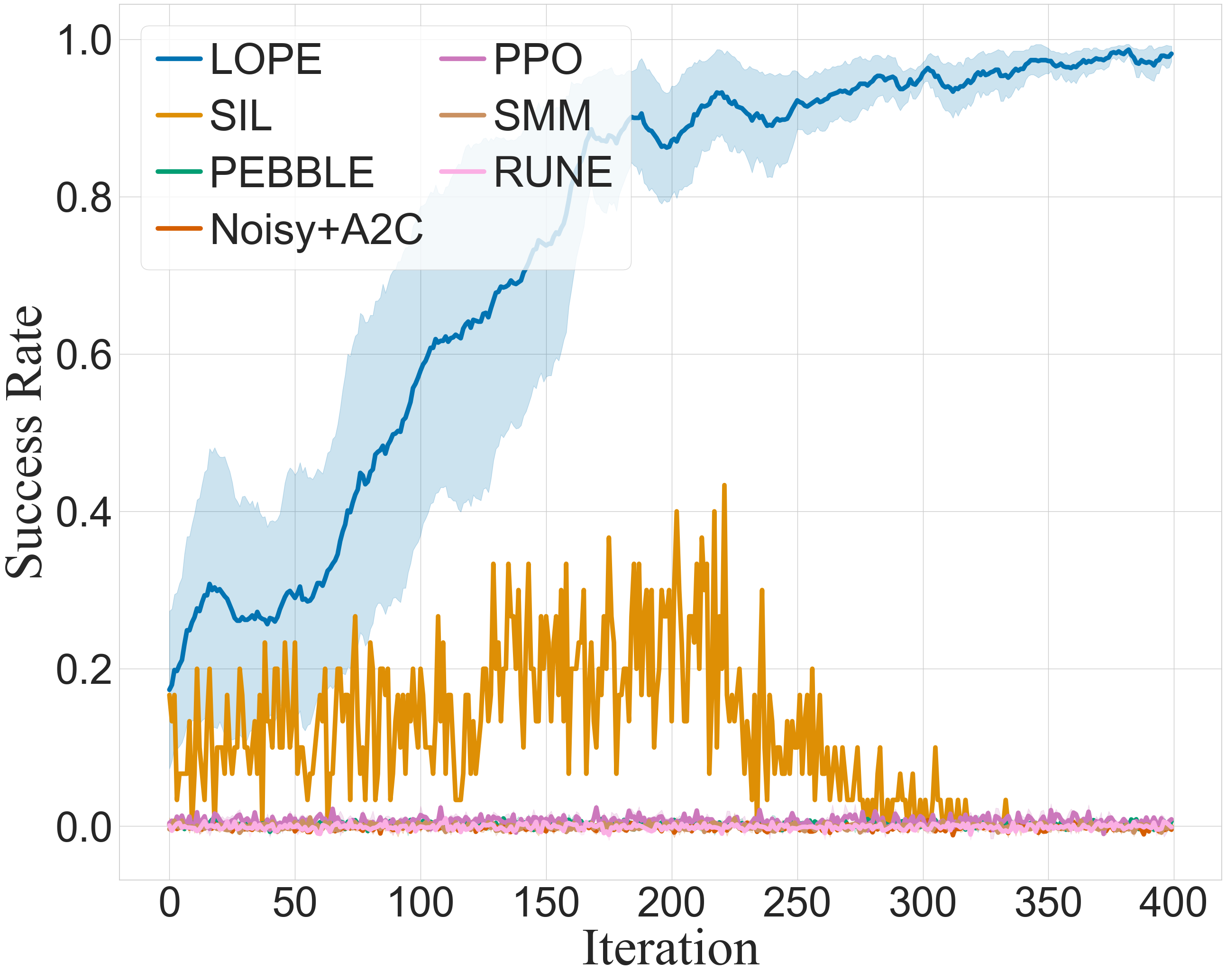}
    }
    \quad
    \subfloat[]{
      \label{fig:medium}
      \includegraphics[width=4.5cm]{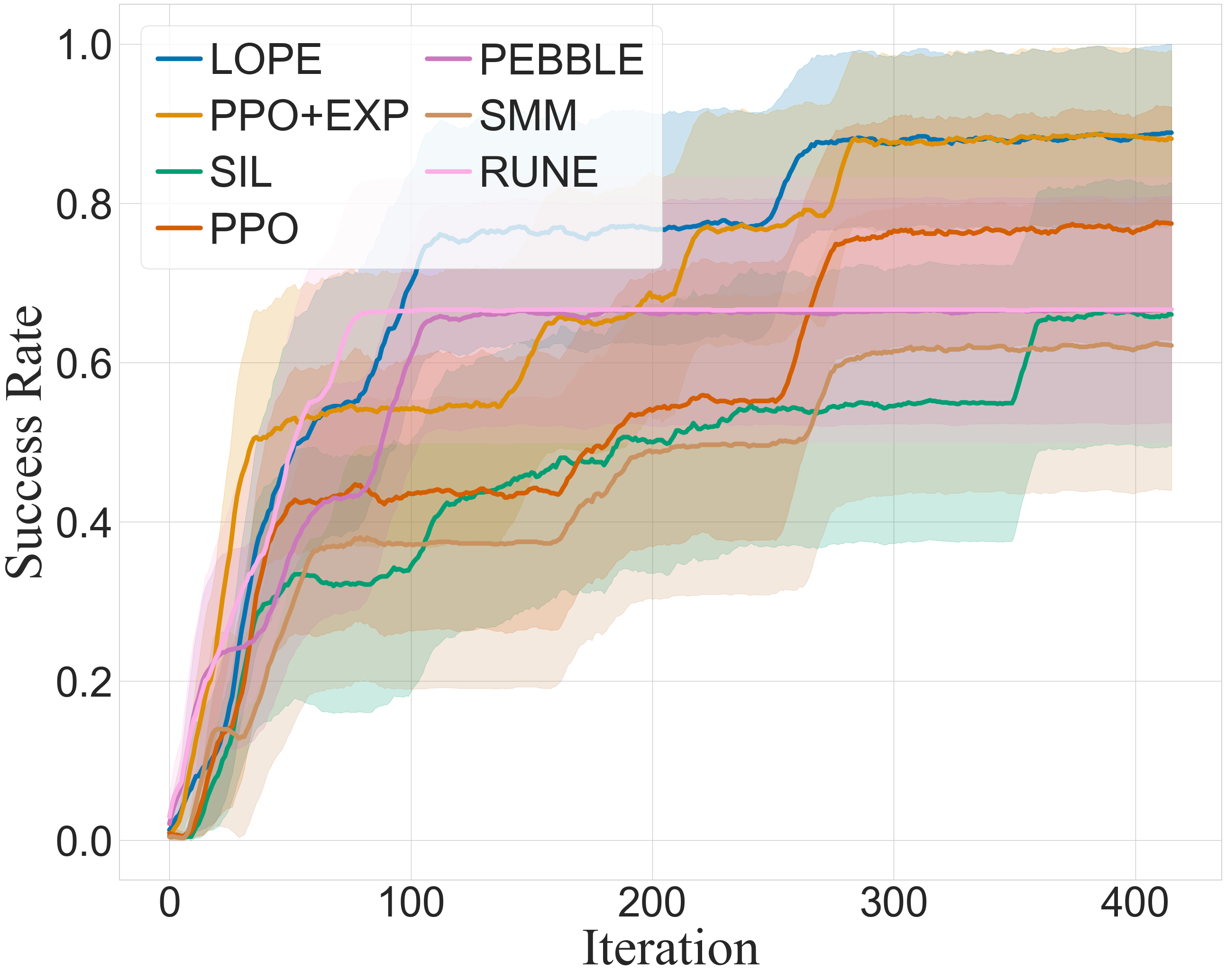}
    }
    \quad
    \subfloat[]{
      \label{fig:large}
      \includegraphics[width=4.5cm]{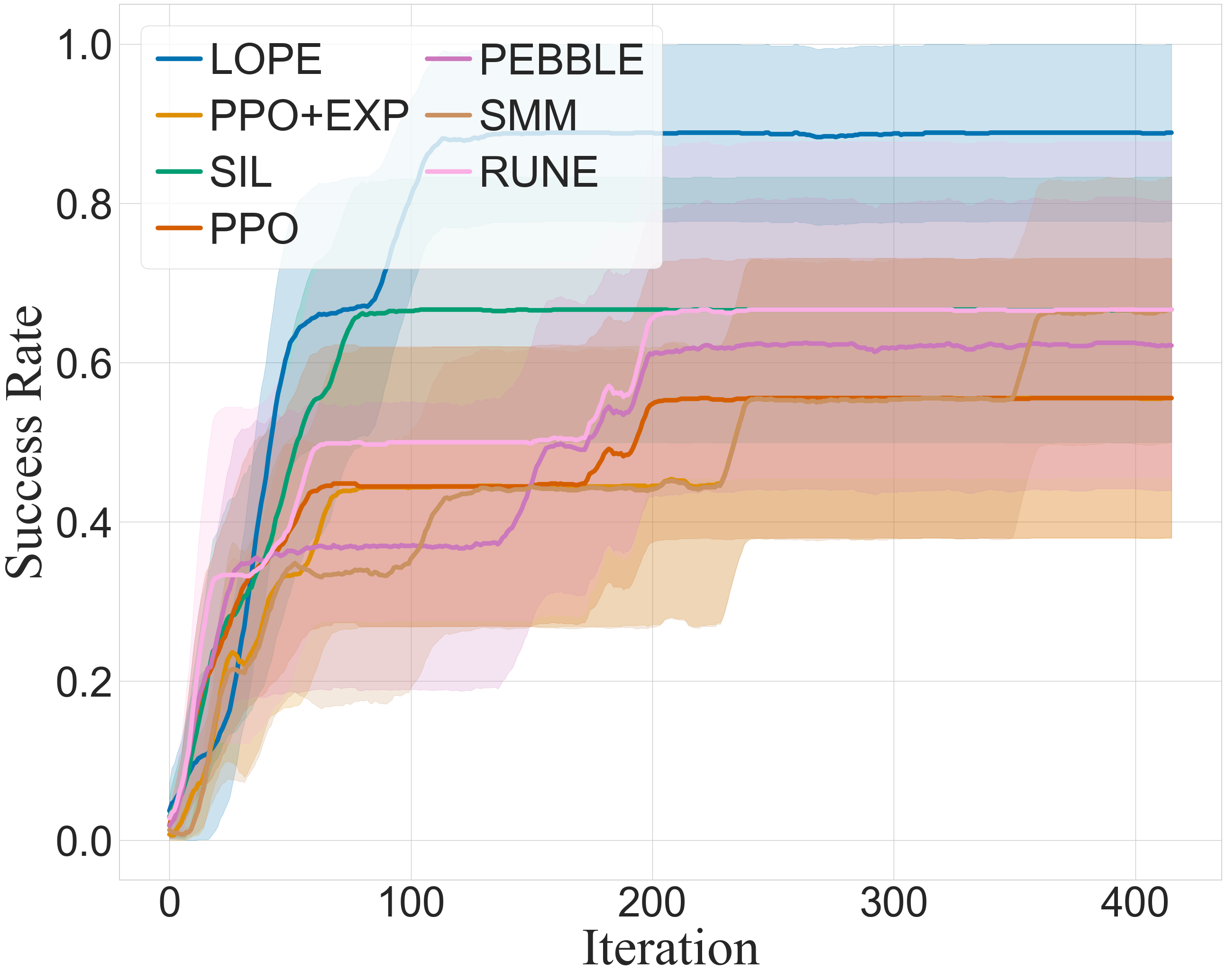}
    }
    \caption{(a) Learning curves of success rate on the continuous maze tasks: (1) AntMaze-Umaze; (2) PointMaze-Medium; (3) PointMaze-Large.}
    \label{fig:ant_exps}
\end{figure*}

We further depicted the actual and human-annotated trajectories in two state visitation graphs. The results are presented in Fig.~\ref{fig:grid}. It implies that the agent can imitate the preference-labeled trajectories perfectly. Therefore, our method can exploit human preferences efficiently and learn the desired policy.

\subsection{Hard-Exploration Continuous Mazes} 
We set three different distance nodes and consider trajectories passing farther distance nodes as preference-labeled trajectories with higher priority in the continuous maze tasks. Specifically, in the ant maze, these nodes are the two corners and the target point. In the PointMaze-Medium task, these nodes are set up in sequence as $\{(3,3), (6,4), (6,6)\}$. Finally, we select $\{(3,6),(5,6),(7,9)\}$ as check nodes in the PointMaze-Large task. Trajectories that passed through these nodes more ``gracefully'' and ``faster'' were preferred. Fig.~\ref{fig:ant_rate} presents the result.
%
%

Compared with the other baseline methods, LOPE can learn faster and achieve higher convergence values of success rate. LOPE has a solid theoretical foundation that provides a guarantee for performance improvement. Specifically, in the AntMaze environment, as shown in Fig.~\ref{fig:ant_rate}, only the LOPE agent can reach the target point and obtain the reward. At the same time, all other baseline methods failed to learn the policy leading to the reward. In the PointMaze tasks, as shown in Figs.~\ref {fig:medium} and~\ref{fig:large}, the LOPE achieves a higher success rate than other baseline approaches. Notably, the SMM method used a policy-level state marginal matching objective to encourage exploration, and LOPE significantly outperformed SMM, demonstrating the effectiveness of the trajectory-wise state marginal matching objective. The results compared with PEBBLE and RUNE indicate that our approach can avoid the potential information bottleneck and fully utilize the data information in preferences.


\subsection{Locomotion Tasks from MuJoCo} 
For easier access to preference-labeled trajectories, we note that SparseHalfCheetah and SparseHopper can only move along the $x$-axis. To determine preference-labeled trajectories, we set up three different nodes roughly evenly within the range of motion of the agent. Trajectories that passed through these nodes more ``gracefully'' and ``faster'' were preferred. Specifically, the term "graceful" is used to describe an agent that can travel a greater distance at each step with less cost, thereby minimizing control costs. The term``fast" indicates that a trajectory can approach the target in fewer steps, thereby minimizing the number of steps.
\begin{figure}[ht]
  \centering
  \subfloat[]{
    \label{fig:cheetah_rate}
    \includegraphics[width=4.2cm]{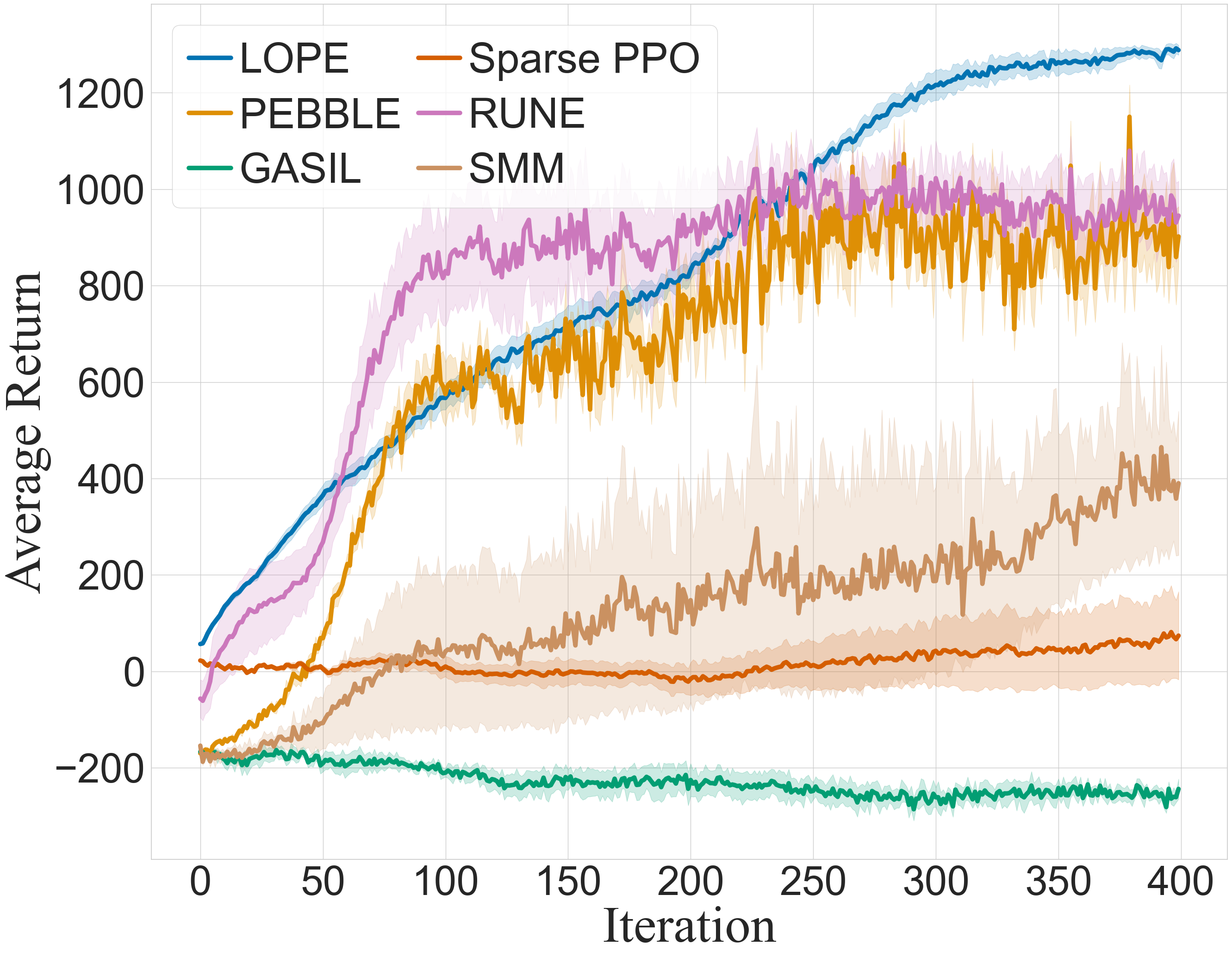}
    }
  \subfloat[]{
    \label{fig:hopper_rate}
    \includegraphics[width=4.2cm]{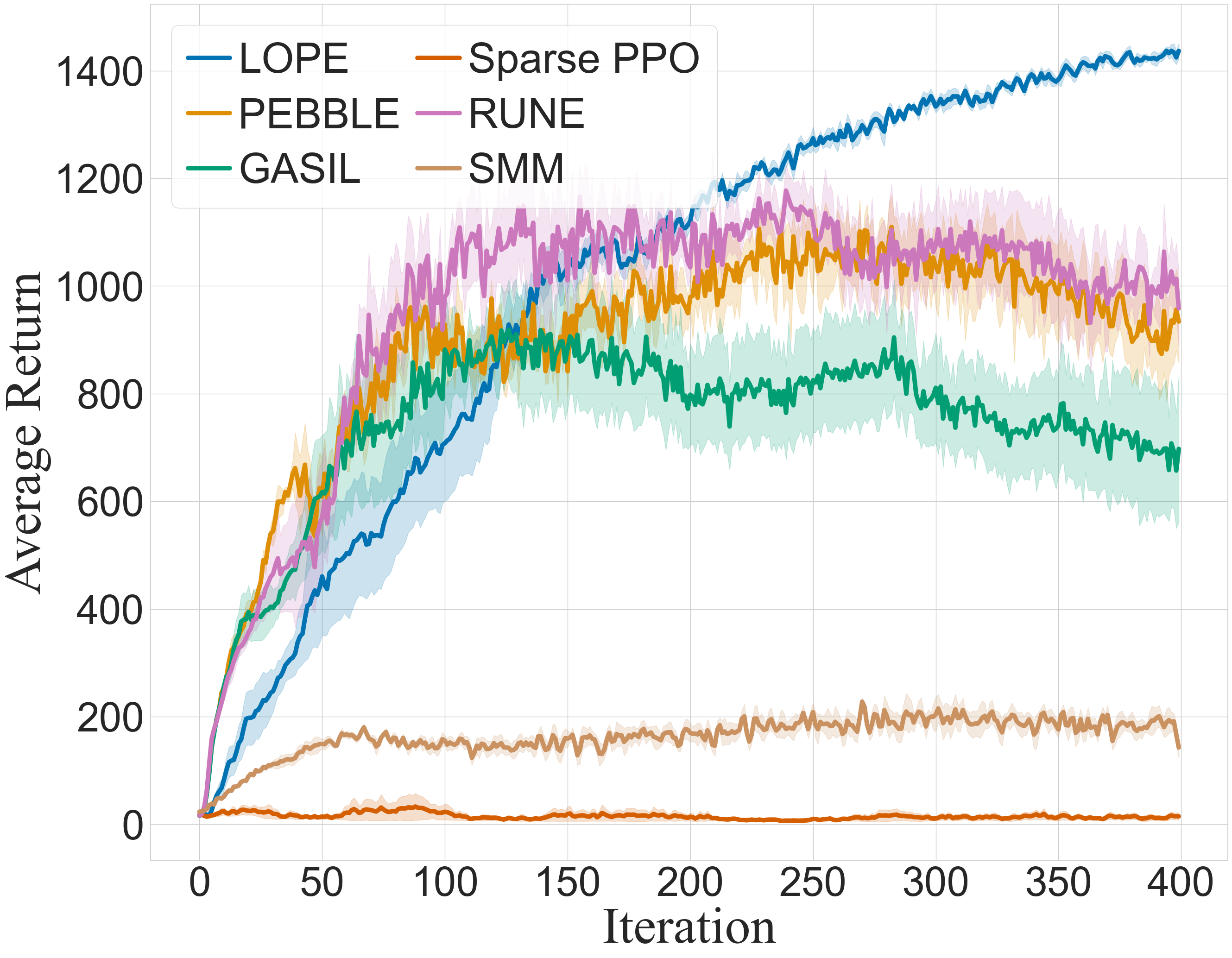}
    }
  \caption{Learning curves of average return on the SparseHalfCheetah and SparseHopper tasks.}
  \label{fig:locomotion_return}
\end{figure}

As shown in Fig.~\ref{fig:locomotion_return}, by selecting the preferred trajectories through the set nodes and considering them as guidance, the LOPE agent can reproduce similar behaviors quickly and gather the survival bonus in both sparse reward tasks. Specifically, in SparseHopper, LOPE achieves performance comparable to PPO with dense rewards, indicating that the policy can approach optimality even when guided only by preferences. Due to the complexity of high-dimensional dynamics, PPO exhibits suboptimal convergence in the HalfCheetah environment. LOPE shows robust performance but also exhibits a plateau, consistent with prior observations. We hypothesize that this is due to limited exploration coverage and propose off-policy extensions to address this in future work. 

On the other hand, although the PbRL methods, PEBBLE and RUNE, exploited human preferences similar to our method, these methods learned the optimal policy more slowly. These approaches do not leverage the information in preferences, such as the agent's coordinate position in the environment, and create significant information bottlenecks when utilizing previous successful experiences. This information bottleneck prevents the agent from directly learning from these demonstrations. The SMM agent performs unsatisfactorily on the SparseHalfCheetah and SparseHopper tasks, further demonstrating the effectiveness of the trajectory-wise state marginal matching objective. Other baseline methods depend on the action- or parameter-space noise sampled from the Gaussian distribution, which is inadequate and inefficient for hard-exploration tasks.

%
%
\begin{figure}[ht]
  \centering
  \subfloat[]{
    \includegraphics[width=4.3cm]{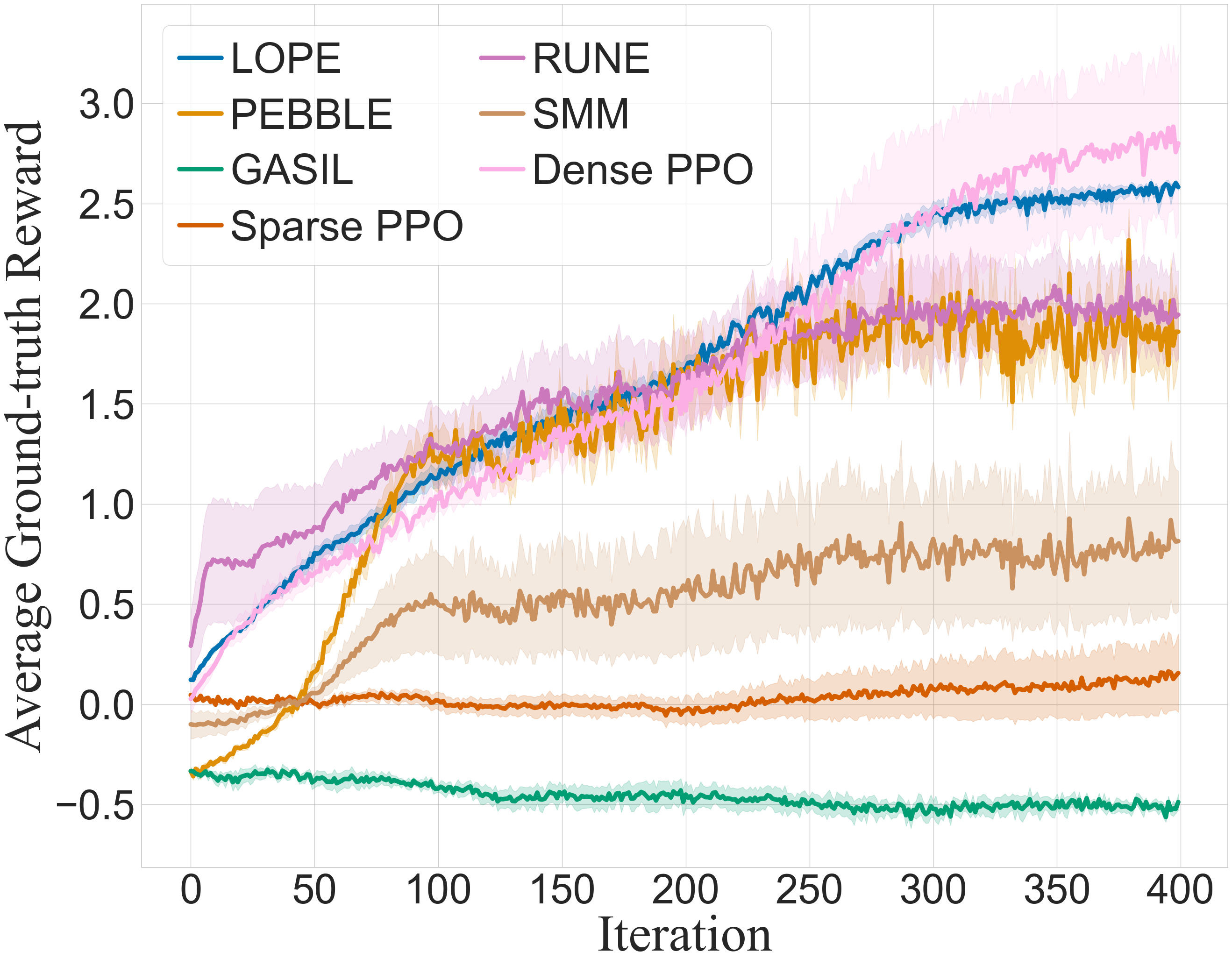}
    \label{fig:cheetah_ground_reward}
    }
    \subfloat[]{
    \includegraphics[width=4.2cm]{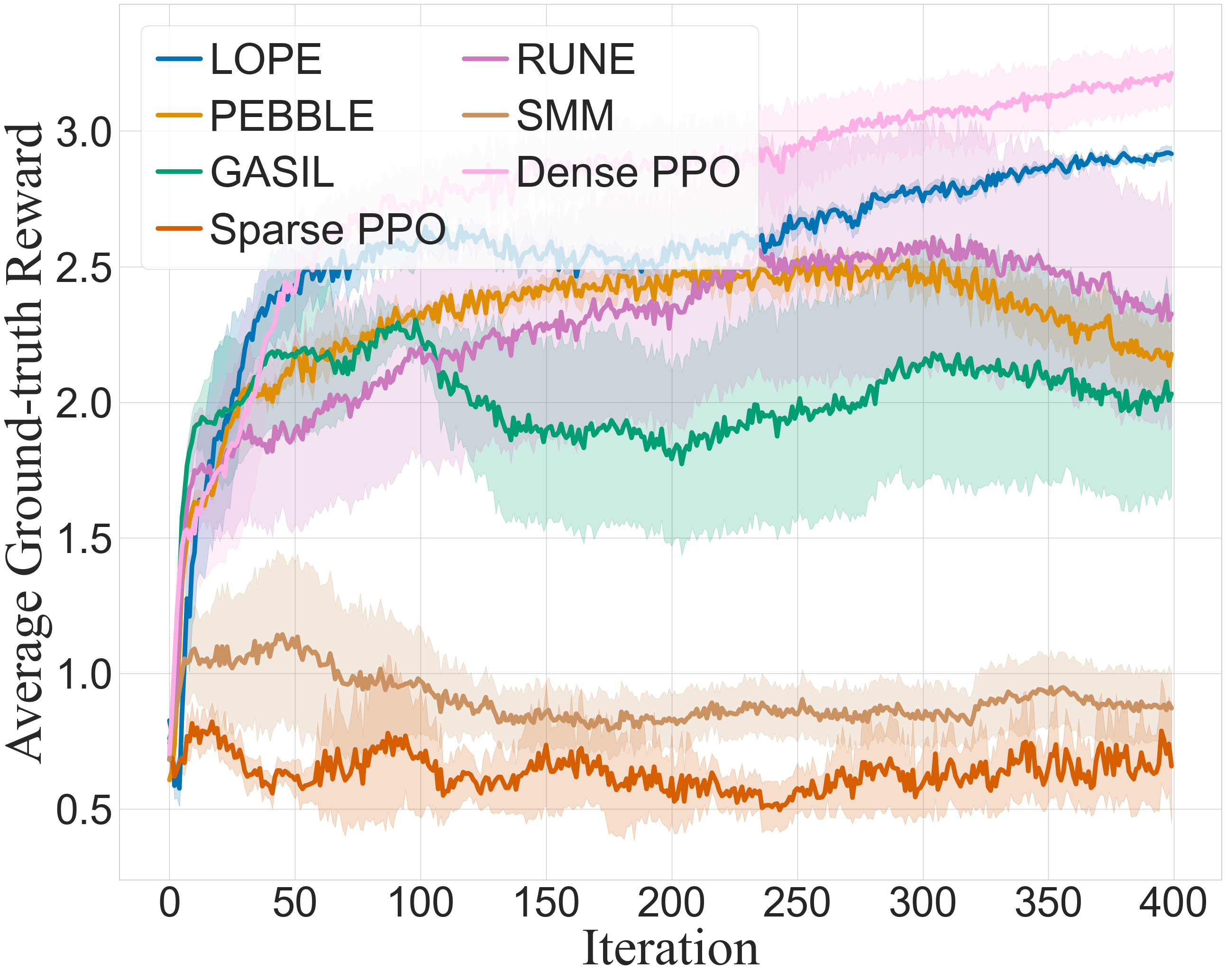}
    \label{fig:hopper_ground_reward}
    }
  \caption{Learning curves of ground-truth rewards on the SparseHalfCheetah and SparseHopper tasks.}
  \label{fig:reward_learning}
\end{figure}

\subsection{Evaluation of Ground-Truth Reward Learning}
A deeper examination of the experimental results is provided to verify the effectiveness of LOPE in learning near-optimal behaviors. In this experiment, we evaluated actions learned by different agents with the ground-truth reward function, which is the default setting of OpenAI Gym~\cite{1606.01540}. We then calculated the average ground-truth reward for each agent and compared them in Fig.~\ref{fig:reward_learning}. For comparison, we also trained PPO with this dense ground-truth reward function. 
\begin{figure}[ht]
\centering
\subfloat[]{
  \includegraphics[width=8.5cm]{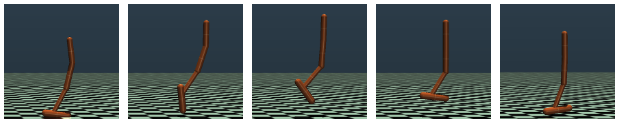}
  \label{fig:cheetah_traj}
}

\subfloat[]{
  \includegraphics[width=8.5cm]{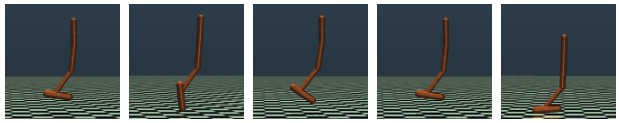}
  \label{fig:hopper_traj}
}
\caption{
  (a) Agent trained by LOPE; (b) Agent trained by PPO with the ground-truth reward function. 
}
\label{fig:trajs_presentation}
\end{figure}

\begin{figure*}[htb]
    \centering
    \subfloat[]{
      \label{fig:ablation}
      \includegraphics[width=4.5cm]{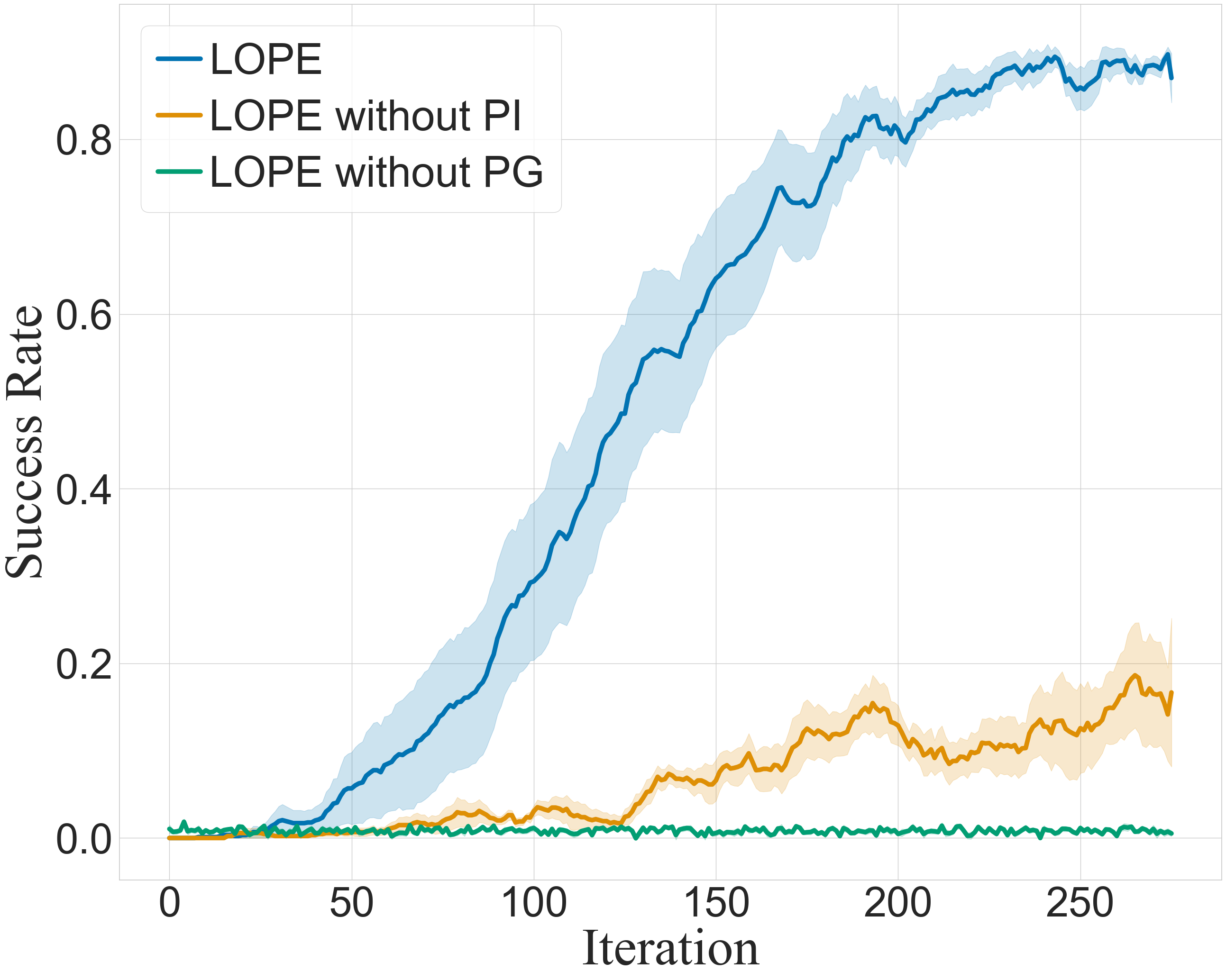}
    }
    \quad
    \subfloat[]{
      \label{fig:mislabel}
      \includegraphics[width=4.5cm]{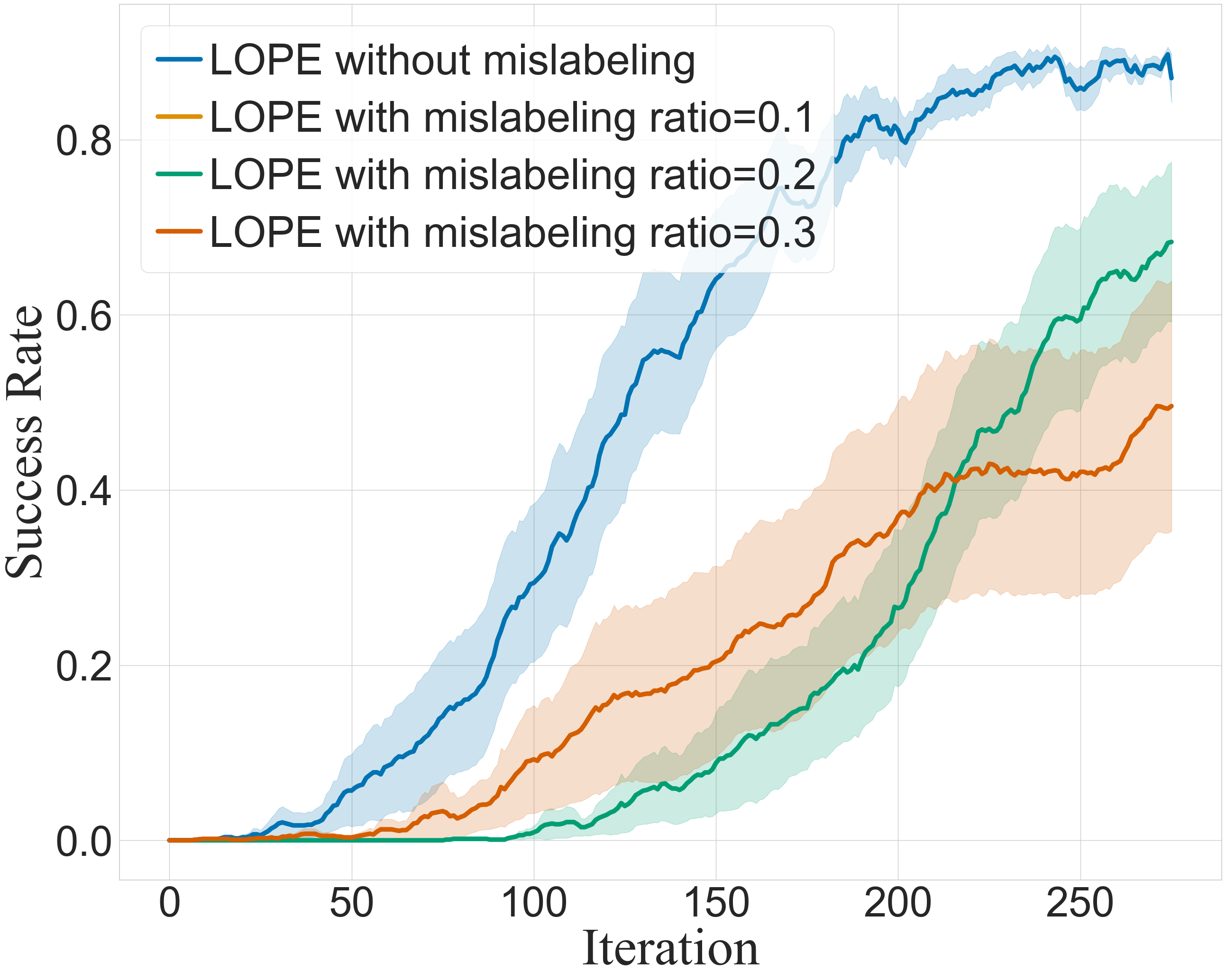}
    }
    \quad
    \subfloat[]{
      \label{fig:exchange}
      \includegraphics[width=4.5cm]{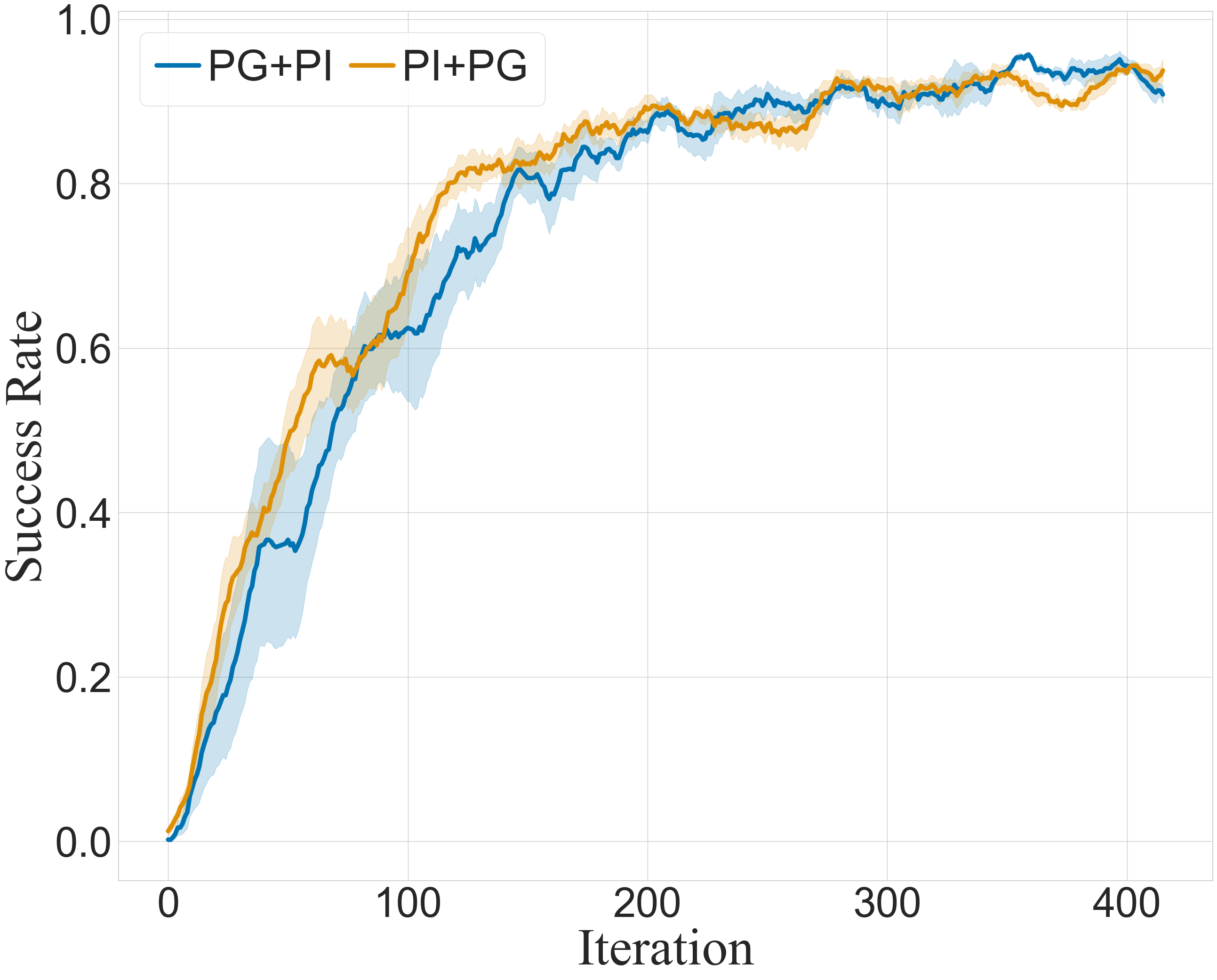}
    }
    \caption{(a) Ablation results of algorithm structure; (2) Performance comparison under mislabeling preferences; (3) Ablation results of exchanging the PI and PG steps.}
    \label{fig:ab_exps}
\end{figure*}

According to Fig.~\ref{fig:reward_learning}, the LOPE's average ground-truth reward dramatically increases at the beginning of training, similar to PPO's learning trend in the default OpenAI Gym reward setting. Meanwhile, the value of the LOPE's average ground-truth reward improves incrementally during policy optimization, and this value gradually approximates that of PPO with the default reward function by the end of training. This result suggests that human feedback can facilitate the agent's learning of near-optimal behaviors that are rewarded with a high score by the default OpenAI Gym reward function. 

To obtain a more intuitive presentation of the trajectories learned by the agent, we display the behaviors of agents trained using LOPE with human preferences and PPO with dense default rewards. Interestingly, LOPE can learn similar behavior patterns compared with those trained with PPO. This fact demonstrates that simple trajectory-wise state marginal matching can guide the agent in generating near-optimal behaviors, thereby avoiding the need for complicated reward engineering.

\subsection{Ablation Analyses}
The experimental results presented in the previous section indicate that LOPE outperforms other baseline approaches on several challenging tasks. We are now interested in whether these advantages still hold when the preference setting changes. We conducted the ablation experiments in five different aspects. We first conducted another ablation study on the algorithm structure to demonstrate the effectiveness of preference guidance. We then designed an ablation experiment and compared the LOPE's performance on preferences with different qualities. Moreover, we exchanged the PI and the PG steps and observed the experimental results. Finally, we considered using different kernel functions and combining the weakly supervised learning technique to test the performance of LOPE. These ablation studies are based on the grid world environment. 

\textbf{LOPE with different structures.} We ran LOPE without the policy improvement step (LOPE w/o PI) and LOPE without the preference guidance step (LOPE w/o PG) and obtained their results separately. The results are presented in Fig.~\ref{fig:ab_exps}\subref{fig:ablation}. LOPE w/o PG cannot find the treasure and learn the optimal policy. In contrast, LOPE w/o PI can find the treasure; however, the learning efficiency of the optimal policy is low. This result implies that preference guidance can enhance the agent's exploration efficiency, while the policy improvement step ensures the reproduction of these successful trajectories.

\textbf{Preferences with different qualities.} In this experiment, preference-labeled trajectories can be mislabeled with different ratios. Specifically, we experimented with synthetic noisy labels in the preference comparator (e.g., by randomly flipping a portion of pairwise feedback with ratios $0, 0.1, 0.2, 0.3$) and obtained corresponding experimental results. The performance comparison is shown in Fig.~\ref{fig:mislabel}. LOPE achieves the highest performance when the mislabeling ratio is 0. Moreover, we observed that LOPE remains relatively stable up to moderate noise levels ($\approx$20\%). These early results suggest that the implicit averaging effect of trajectory-level marginal matching provides a degree of robustness against noisy labels. 

\textbf{Exchanging the PI and PG steps.} We exchanged the PI and PG steps of LOPE and compared their performance in the grid-world task with random goals. "PI+PG" represents the LOPE version that performs the PI step first, while "PG+PI" represents the opposite. The results are presented in Fig.~\ref{fig:exchange}. The experimental result illustrates that these two methods can achieve approximately the same final performance. However, the PI+PG version learns faster than PG+PI in the early stage of the training, likely because updating the policy with stronger environmental signals (PI step) first establishes a better exploration foundation, upon which human preferences can more effectively guide further behavior..

\textbf{Ablation experiment of weakly supervised learning.}
We combined the weakly supervised learning technique to leverage unlabeled data to assist in the policy optimization process in the initial stage of training. A cross-entropy (CE) loss is added to the policy improvement objective function because of limited supervised signal information. The experimental result is shown in Fig.~\ref{fig:supervised}. This result demonstrates that the weakly supervised learning technique makes a significant contribution to improving the algorithm's performance during the early stages of training. A possible reason is that the weakly supervised learning loss function can increase the likelihood of the preferred action, encouraging the agent to visit promising regions with higher rewards.
\begin{figure}
    \centering
    \subfloat[]{
    \label{fig:supervised}
    \includegraphics[width=4.2cm]{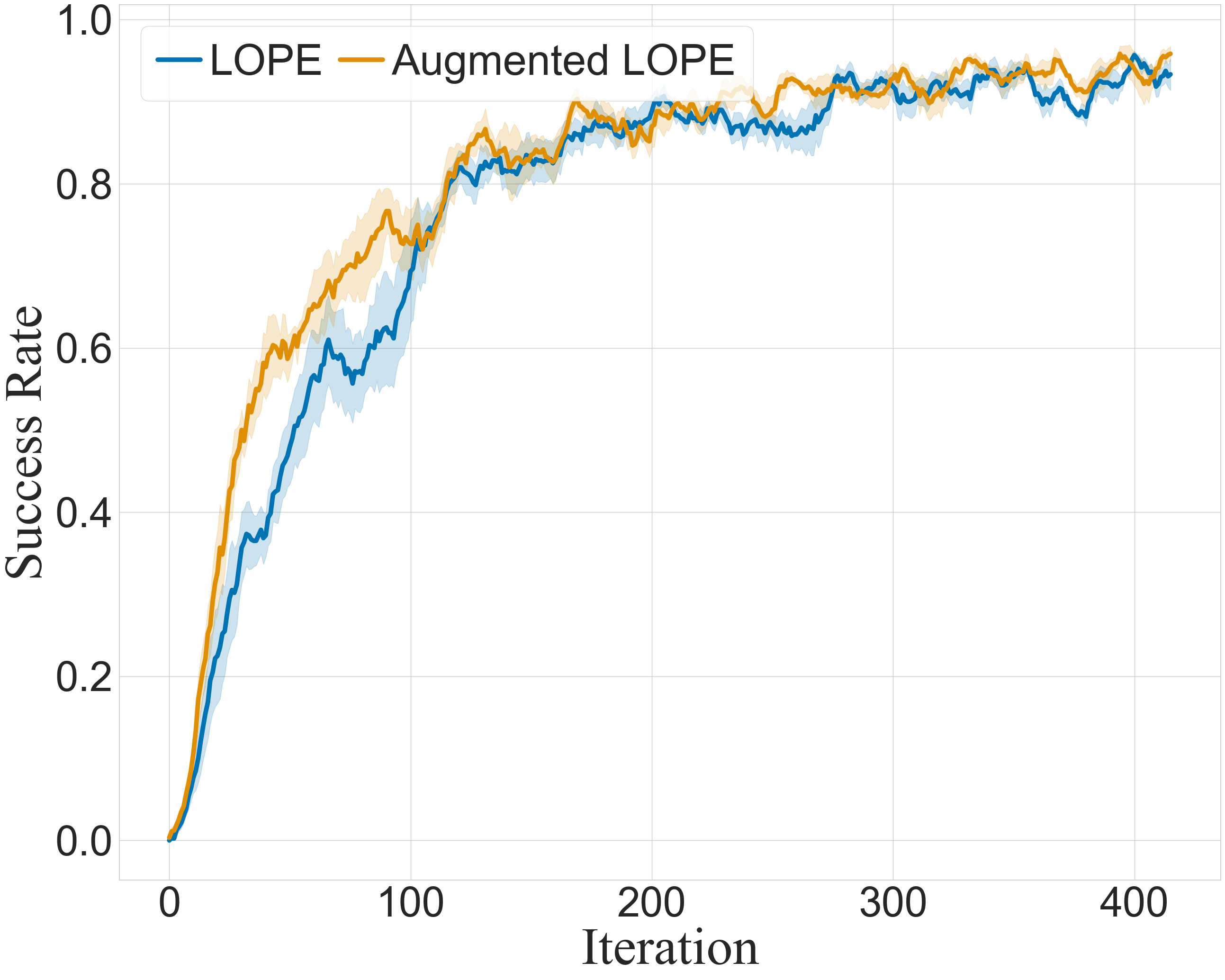}
    }
    \subfloat[]{
    \label{fig:kernel}
    \includegraphics[width=4.2cm]{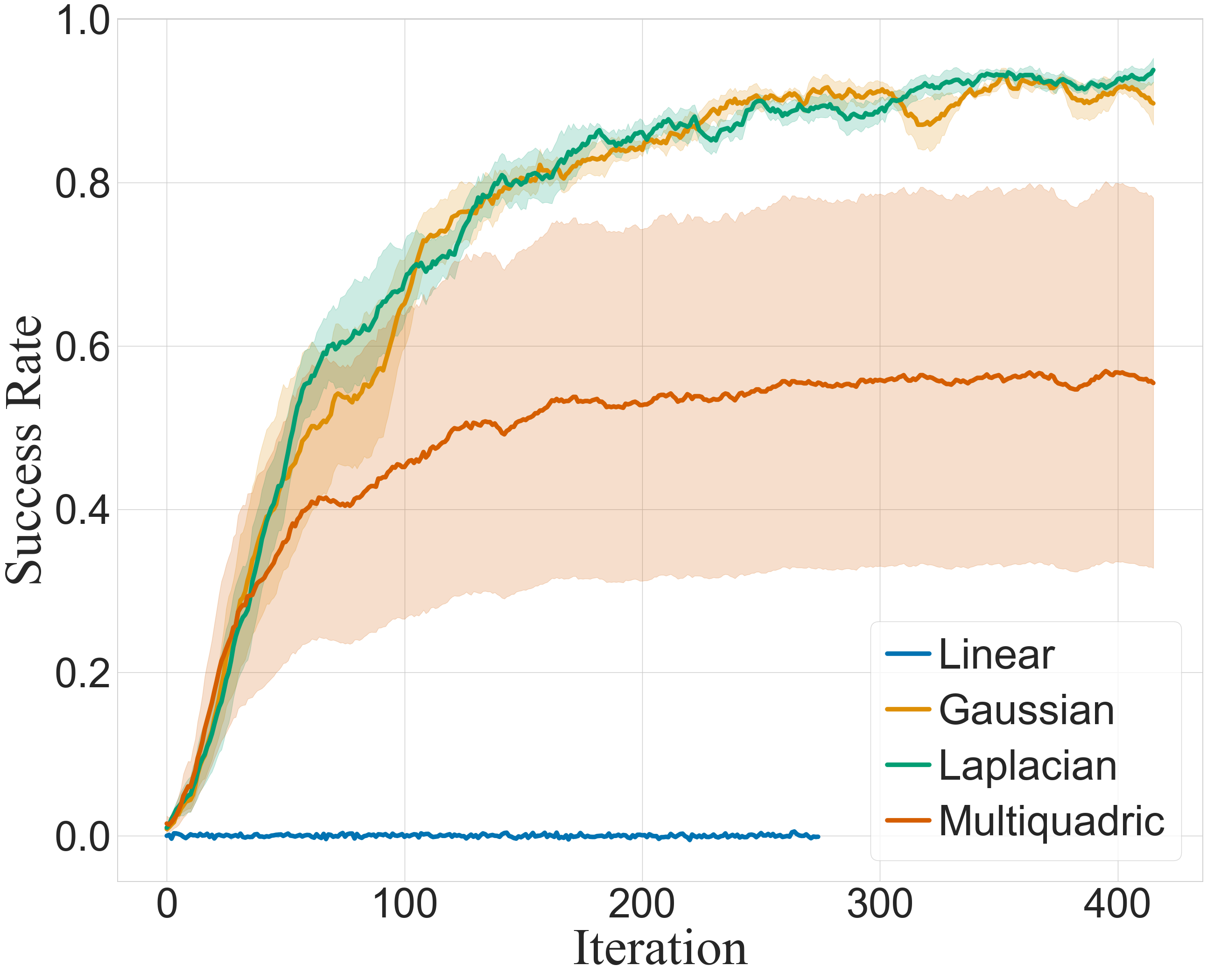}
    }
    \caption{(a) The performance comparison between the original LOPE and the LOPE augmented with weakly supervised learning; (b) Ablation results of different kernel functions.}
\end{figure}

\textbf{Ablation results of different kernel choices.}
We conducted an ablation study to assess the impact of different kernel functions used in the MMD-based preference-guided state marginal matching objective. Specifically, we evaluated four representative kernels: linear, Gaussian, Laplacian, and multiquadric. As shown in Fig.~\ref{fig:kernel}, the Laplacian kernel delivers the most favorable performance, achieving both fast convergence and high final success rate. The Gaussian kernel follows closely, exhibiting similarly strong results. The multiquadric kernel leads to slower learning and plateaus at a lower success rate, indicating poor generalization to preferred trajectories. Notably, the linear kernel completely fails to guide policy learning, resulting in zero improvement throughout training. This highlights the importance of nonlinearity and localized sensitivity in kernel selection. Overall, the experiment suggests that Laplacian and Gaussian kernels are most effective for modeling trajectory-level preferences under our marginal matching framework.

\section{Conclusion}
\label{sec:conclusion}
This study develops a practical PbRL algorithm for hard-exploration tasks with long horizons and sparse rewards. Our key idea is to adjust the agent's exploration by considering preferred trajectories as guidance, thereby avoiding the need to learn a reward function. LOPE involves a two-step sequential policy optimization process, including trust-region-based policy improvement and preference guidance steps. Preference guidance is reformulated as a novel trajectory-wise state marginal matching problem that minimizes a novel maximum mean discrepancy distance between the preferred trajectories and the learned policy. We further provide a theoretical analysis to characterize the performance improvement bound and evaluate the effectiveness of the proposed approach. This result implies convergence to a near-optimal policy under mild assumptions on the preference quality and KL constraints, thus offering both performance insight and theoretical justification of LOPE’s behavior. Extensive experimental results verify the superior performance of LOPE over different competitive baseline methods on several sparse-reward benchmark tasks. 

In future work, we aim to enhance the robustness and generality of LOPE further. One important direction is to enhance the system’s reliability in the face of noisy, conflicting, or low-quality human preferences. While LOPE exhibits preliminary resilience under moderate synthetic noise, incorporating uncertainty-aware preference modeling (e.g., confidence-weighted aggregation or disagreement filtering) could significantly improve robustness in real-world deployments. Another promising extension is to integrate multimodal human feedback, such as natural language descriptions or voice instructions, to enrich the agent’s understanding of task intent. LOPE’s modular framework naturally accommodates these enhancements, and we consider these challenges to be important next steps toward building more practical and human-aligned preference-based reinforcement learning systems.



{\appendices
\section{Proof of Lemma~\ref{lem:pref_guide}}
\label{sec:proof_lem1}
\begin{proof}
  The objective function in Eq.~\eqref{eq:perference_guidance} can be transformed as follows:
  \begin{equation}
    \begin{aligned}
      \label{eq:expanded_pg_expectation}
      \underset{\tau \sim \rho_\theta}{\mathbb{E}}\left[\underset{\upsilon \sim \mathcal{P}}{\mathbb{E}} \left[d(\tau, \upsilon)\right]\right] &= \underset{\tau \sim \rho_\theta}{\mathbb{E}}\left[\underset{\upsilon \sim \mathcal{P}}{\mathbb{E}} \left[\underset{s_\tau \sim \rho_\tau}{\mathbb{E}} \left[\operatorname{dist}(s_\tau, \upsilon)\right]\right]\right]\\
      &= \underset{\tau \sim \rho_\theta}{\mathbb{E}}\left[\underset{s_\tau \sim \rho_\tau}{\mathbb{E}}\left[\underset{\upsilon \sim \mathcal{P}}{\mathbb{E}} \left[\operatorname{dist}(s_\tau, \upsilon)\right]\right]\right],
    \end{aligned}
  \end{equation}
  where $d_\theta$ is the discounted state visitation distribution defined in Section~\ref{sec:Preliminaries}. In this manner, $\underset{\upsilon \sim \mathcal{P}}{\mathbb{E}} \left[\operatorname{dist}(s_\tau, \upsilon)\right]$ in Eq.~\eqref{eq:expanded_pg_expectation} can be viewed as an intrinsic reward and let $r_g(s, a)=\underset{\upsilon \sim \mathcal{P}}{\mathbb{E}} \left[\operatorname{dist}(s_\tau, \upsilon)\right]$. Then, we obtain 
  \begin{equation}
    \begin{aligned}
      \label{eq:rg_expectation}
      \underset{\begin{subarray}{l} s \sim d_\theta \\ a \sim \pi_\theta \end{subarray}}{\mathbb{E}} \left[r_g(s,a)\right] = \underset{\tau \sim \rho_\theta}{\mathbb{E}}\left[\underset{\upsilon \sim \mathcal{P}}{\mathbb{E}} \left[d(\tau, \upsilon)\right]\right].
    \end{aligned}
  \end{equation}
  By substituting the original objective function in Eq.~\eqref{eq:perference_guidance} with the above Eq.~\eqref{eq:rg_expectation}, we obtain Eq.~\eqref{eq:expanded_perference_guidance}. Hence, this MMD divergence minimization problem in Eq.~\eqref{eq:constraint} is reduced into a policy-gradient optimization algorithm with intrinsic rewards learned from human feedback. Then, the gradient of the objective function in Eq.~\eqref{eq:expanded_perference_guidance} for parameters $\theta$ is derived as follows:
  \begin{equation}
    \label{equ:nabla_E_D_mmd}
    \begin{aligned}
      g = \underset{\begin{subarray}{l} s \sim d_\theta \\ a \sim \pi_\theta \end{subarray}}{\mathbb{E}}
      \left[\nabla_{\theta}\log\pi_{\theta}(a \vert s)Q_g(s,a)\right],
    \end{aligned}
  \end{equation}
  where
  \begin{equation}
    \label{equ:Q}
    Q_g(s_t, a_t) = \underset{s_{t+1},a_{t+1},\dots}{\mathbb{E}}\left[\sum_{l=0}^{\infty}\gamma^{l}r_g(s_{t+l}, a_{t+l})\right].
  \end{equation}
\end{proof}

\section{Proof of Policy Performance Bound}
\label{sec:proof_bound}
\begin{definition}[$\alpha$-coupled policies~\cite{schulman2015trust,kang2018policy}]
\label{def:coupled}
    $(\pi, \tilde{\pi})$ are $\alpha$-coupled policies if they defines a joint probability distribution $(a, \tilde{a})\vert s$, such that $P(a\neq \tilde{a}\vert s)\le \alpha$ for all $s$. $\pi$ and $\tilde{\pi}$ will denote the marginal distributions of $a$ and $\tilde{a}$, respectively.
\end{definition}

Given two arbitrary policies $\pi_1$ and $\pi_2$, define the total variation discrepancy as $D_{\rm TV} (\pi_1,\pi_2)[s] = (1/2) \sum_a |\pi_1(a|s) - \pi_2(a|s)|$ and $D_{\rm TV}^{\max}(\pi_1, \pi_2)=\max_{s\in\mathcal{S}}D_{\rm TV} (\pi_1,\pi_2)[s]$. 
\begin{lemma}[Adopted from~\cite{levin2017markov}] 
\label{lem:TV_alpha}
Assume $p_X$ and $p_Y$ are probability distributions with $D_{\rm TV}(p_X, p_Y)=\alpha$, then there exists a joint probability distribution (X, Y) whose marginal distributions are $p_X$, $p_Y$, where $X=Y$ with probability $1-\alpha$.
\end{lemma}
\begin{lemma}[Adopted from~\cite{schulman2015trust,achiam2017constrained}]
  \label{lem:performancediff}
  Suppose $\pi_1$, $\pi_2$ are two stochastic policies defined on $\mathcal{S}\times\mathcal{A}$, we have:
  \begin{equation}
      J(\pi_1) = J(\pi_2) + \mathbb{E}_{\tau\sim\rho_{\pi_1}}\left[\sum_{t=0}^\infty \gamma^t A_{\pi_2}(s,a)\right].
  \end{equation}
\end{lemma}
The proof details can be found in~\cite{schulman2015trust}, and we omit them here for brevity. Given two arbitrary stochastic policies $\pi_1$ and $\pi_2$, define the expected advantage of the policy $\pi_1$ over the policy $\pi_2$ at state $s$ as $A^{\pi_1\vert\pi_2}(s)$:
\begin{equation}
    A^{\pi_1\vert\pi_2}(s) = \mathbb{E}_{a\sim\pi_1(\cdot\vert s)}\left[A_{\pi_2}(s,a)\right].
\end{equation}


According to Lemma~\ref{lem:performancediff}, we can derive the following conclusion. Given three arbitrary stochastic policies $\pi$, $\tilde{\pi}$, $\pi_b$, then the following equation holds:
\begin{equation}
    \label{eq:J_diff_1}
    \begin{aligned}
        &J(\tilde{\pi}) - J(\pi_b) \\ 
        = &J(\pi_b) + \underset{\tau\sim\rho_{\tilde{\pi}}}{\mathbb{E}}\left[\sum^\infty_{t=0}\gamma^t A^{\tilde{\pi}\vert\pi_b}(s_t)\right] - J(\pi_b)\\
        = &\underset{\tau\sim\rho_{\tilde{\pi}}}{\mathbb{E}}\left[\sum^\infty_{t=0}\gamma^t A^{\tilde{\pi}\vert\pi_b}(s_t)\right].
    \end{aligned}
\end{equation}
Here, Eq.~\eqref{eq:J_diff_1} is obtained by defining the expected return $J(\tilde{\pi})$ of $\tilde{\pi}$ over $\pi_b$ as $J(\pi_b) + \mathbb{E}_{\tilde{\pi}}\left[\sum^\infty_{t=0}\gamma^t A^{\tilde{\pi}\vert\pi_b}(s_t)\right]$. If we define the expected returns of $\tilde{\pi}$ and $\pi$ over $\pi$, then we have:
\begin{align}
    \label{eq:tilde_J_over_pi} J(\tilde{\pi}) &= J(\pi) + \mathbb{E}_{\tau\sim\rho_{\tilde{\pi}}}\left[\sum^\infty_{t=0}\gamma^t A^{\tilde{\pi}\vert\pi}(s_t)\right],\\
    \label{eq:J_b_over_pi} J(\pi_b) &= J(\pi) + \mathbb{E}_{\tau\sim\rho_{b}}\left[\sum^\infty_{t=0}\gamma^t A^{\pi_b\vert\pi}(s_t)\right].
\end{align}
Substitute Eqs.~\eqref{eq:tilde_J_over_pi} and~\eqref{eq:J_b_over_pi} into $J(\tilde{\pi}) - J(\pi_b)$, we have
\begin{equation}
\label{eq:J_diff_2}
    \begin{aligned}
        &J(\tilde{\pi}) - J(\pi_b) \\ 
        = &\underset{\tau\sim\rho_{\tilde{\pi}}}{\mathbb{E}}\left[\sum^\infty_{t=0}\gamma^t A^{\tilde{\pi}\vert\pi}(s_t)\right] - \underset{\tau\sim\rho_{b}}{\mathbb{E}}\left[\sum^\infty_{t=0}\gamma^t A^{\pi_b\vert\pi}(s_t)\right].
    \end{aligned}
\end{equation}

Combining Eqs.~\eqref{eq:J_diff_1} and~\eqref{eq:J_diff_2} gives
\begin{equation}
    \begin{aligned}
        &J(\tilde{\pi}) - J(\pi_b) \\ 
        = &\underset{\tau\sim\rho_{\tilde{\pi}}}{\mathbb{E}}\left[\sum^\infty_{t=0}\gamma^t A^{\tilde{\pi}\vert\pi_b}(s_t)\right]\\
        = &\underset{\tau\sim\rho_{\tilde{\pi}}}{\mathbb{E}}\left[\sum^\infty_{t=0}\gamma^t A^{\tilde{\pi}\vert\pi}(s_t)\right] - \underset{\tau\sim\rho_{b}}{\mathbb{E}}\left[\sum^\infty_{t=0}\gamma^t A^{\pi_b\vert\pi}(s_t)\right].
    \end{aligned}
\end{equation}
A similar result can be found in~\cite{kang2018policy}.

\begin{lemma}[Adopted from~\cite{kang2018policy}]
\label{lem:adv_bound}
    Given that $\pi_1$, $\pi_2$ are $\alpha$-coupled policies, for any state $s$, then we have:
    \begin{equation}
    \left\vert\underset{s\sim\pi_1}{\mathbb{E}}\left[A^{\pi_1\vert\pi_2}(s)\right]\right\vert \le 2\alpha(1-(1-\alpha)^t)\max_{s,a}\vert A_{\pi_2}(s,a)\vert.
    \end{equation}
\end{lemma}

The proof details of Lemma~\ref{lem:adv_bound} can be found in~\cite{kang2018policy}, and we omit them here.

\firstass*

Now, we will derive the performance improvement bound described in Lemma~\ref{lem:assumption-improvement}. The proof of Lemma~\ref{lem:assumption-improvement} is described below. 

\begin{proof}[Proof of Lemma~\ref{lem:assumption-improvement}]
    Since $\beta=D_{\rm KL}^{\max}(\tilde{\pi}, \pi_b)=\max_s D_{\rm KL}(\tilde{\pi}(\cdot\vert s), \pi_b(\cdot\vert s))$, considering $D_{\rm TV}^2(p, q)\le D_{\rm KL}(p, q)$~\cite{pollard2000asymptopia}, we assume $\alpha=D_{\rm TV}^{\max}(\tilde{\pi}, \pi_b)$ and $\alpha$ satisfies $\alpha^2\le\beta$. According to Lemma~\ref{lem:TV_alpha}, $\pi_b$, $\tilde{\pi}$ are $\alpha$-coupled. Furthermore, it is worth mentioning that the expert policy $\pi_b$ satisfies Assumption~\ref{ass:1}, which means
    \begin{equation}
        \mathbb{E}_{a\sim\pi_b}\left[A_{k}(s,a)\right]\ge\Delta>0
    \end{equation}
    Consider introducing $J(\pi_b)$ into $J(\tilde{\pi})$ and $J(\pi)$ and applying Lemma~\ref{lem:performancediff} to $\tilde{\pi}$ and $\pi_b$, then we have 
    \begin{align}
        &J(\tilde{\pi}) - J(\pi)\\
        = &J(\tilde{\pi}) - J(\pi_b) + J(\pi_b) -J(\pi) \\ 
        = &\underset{\tau\sim\rho_{\tilde{\pi}}}{\mathbb{E}}\left[\sum^\infty_{t=0}\gamma^t A^{\tilde{\pi}\vert\pi}(s_t)\right] - \underset{\tau\sim\rho_{b}}{\mathbb{E}}\left[\sum^\infty_{t=0}\gamma^t A^{\pi_b\vert\pi}(s_t)\right]\\
        &+ \underset{\tau\sim\rho_{b}}{\mathbb{E}}\left[\sum^\infty_{t=0}\gamma^t A^{\pi_b\vert\pi}(s_t)\right]\\
        = &\underset{\tau\sim\rho_{\tilde{\pi}}}{\mathbb{E}}\left[\sum^\infty_{t=0}\gamma^t A^{\tilde{\pi}\vert\pi_b}(s_t)\right] + \underset{\tau\sim\rho_{b}}{\mathbb{E}}\left[\sum^\infty_{t=0}\gamma^t A^{\pi_b\vert\pi}(s_t)\right]\\
        \ge &-\sum^\infty_{t=0}\gamma^t\left\vert\underset{\tau\sim\rho_{\tilde{\pi}}}{\mathbb{E}} A^{\tilde{\pi}\vert\pi_b}(s_t)\right\vert + \underset{\tau\sim\rho_{b}}{\mathbb{E}}\left[\sum^\infty_{t=0}\gamma^t A^{\pi_b\vert\pi}(s_t)\right]\\
        \ge &-\sum^\infty_{t=0}\gamma^t\left(2\alpha(1-(1-\alpha)^t)\max_{s,a}\vert A_{\pi_b}(s,a)\vert - \Delta\right) \\
        = &-\frac{2\alpha^2\gamma\epsilon_b}{(1-\gamma)(1-\gamma(1-\alpha))} +\frac{\Delta}{1-\gamma}\\
        \ge &-\frac{2\alpha^2\gamma\epsilon_b}{(1-\gamma)^2} +\frac{\Delta}{1-\gamma},
    \end{align}
    where $\epsilon_b = \max_{s,a}\vert A_b(s,a)\vert$ is the maximum absolute value of $\pi_b$'s advantage function $A_b(s,a)$. Noting $\beta=D_{\rm KL}^{\max}(\tilde{\pi}, \pi_b)$ and applying $D_{\rm TV}^2(p, q)\le D_{\rm KL}(p, q)$~\cite{pollard2000asymptopia} gives
    \begin{equation}
        J(\tilde{\pi}) - J(\pi) \ge -\frac{2\beta\gamma\epsilon_b}{(1-\gamma)^2} +\frac{\Delta}{1-\gamma}.
    \end{equation}
\end{proof}

\begin{table*}[htb]
  \centering
  \caption{Hyperparameter choices of LOPE for different tasks.}
    \begin{center}
      \begin{small}
        \begin{tabular}{lcccccc}
          \toprule
          \bf{Hyperparameter} & \bf{GridWorld} & \bf{AntMaze} & \bf{HalfCheetah} & \bf{Hopper} & \bf{PointMaze-Medium} & \bf{PointMaze-Large} \\
          \midrule
          Value Function Learning Rate & $0.01$ & $0.001$ & $0.001$ & $0.001$ & 1.2e-4 & 1.2e-4\\
          Policy Learning Rate & $0.000015$ & $0.00009$ & $0.00009$ & $0.0004$ & 1.8e-5 & 1.8e-5\\
          Actor Regularization & None & None & None & None & None & None  \\
          Policy Optimizer & Adam & Adam & Adam & Adam & Adam & Adam\\
          Value Function Optimizer & Adam & Adam & Adam & Adam & Adam & Adam \\
          Max Length per Episode & $240$ & $750$ & $500$ & $500$ & 250 & 600\\ 
          Episodes per iteration & $8$ & $30$ & $20$ & $20$ & 12 & 12\\
          Epochs per iteration & $80$ & $65$ & $80$ & $80$ & 65 & 65\\
          Discount Factor & $0.99$ & $0.99$ & $0.99$ & $0.99$ & 0.99 & 0.99\\
          Normalized Observations & False  & False  & False  & False & False  & False \\
          Gradient Clipping & False & False & False & False & False & False \\
          Initial Exploration Policy & None  & $N(0, 0.17)$ & $N(0, 0.15)$ & $N(0, 0.65)$ & $N(0, 0.65)$ & $N(0, 0.75)$\\ 
          Clipped Epsilon & $0.3$ & $0.21$ & $0.2$ & $0.27$ & 0.4 & 0.4\\ 
          \bottomrule
        \end{tabular}
      \end{small}
    \end{center}
    \label{table:parameter}
\end{table*}

Finally, we describe the proof of Theorem~\ref{thm:bound}.
\begin{proof}[Proof of Theorem~\ref{thm:bound}]
  Combining Proposition~\ref{prop:policy_improve} and Lemma~\ref{lem:assumption-improvement}, we come to conclude.
\end{proof}

\section{Neural Network Architectures and Hyperparameters}
\label{sec:nn_parameters}
\subsection{Neural Network Architectures}
LOPE policy architecture for discrete control tasks:
\begin{verbatim}
  (state dim, 64)
  Tanh 
  (64, 64)
  Tanh 
  (64, action num) 
  Softmax 
\end{verbatim}

LOPE policy architecture for continuous control and navigation tasks:
\begin{verbatim}
  (state dim, 64)
  Tanh 
  (64, 64)
  Tanh 
  (64, action dim) 
  Tanh
\end{verbatim}

LOPE value function architecture:
\begin{verbatim}
  (state dim, 64)
  Tanh 
  (64, 64)
  Tanh 
  (64, 1) 
\end{verbatim}

\subsection{Hyper-parameter settings}
The hyperparameter settings for different tasks are presented in Table~\ref{table:parameter}.

\section{Additional Implementation Details}
\label{sec:details}
For clarity in the presentation, we omit certain implementation details about experiment environment settings in Section~\ref{sec:setup} and describe them here. The locomotion control and maze navigation tasks were implemented using MuJoCo~\cite{todorov2012mujoco} and D4RL~\cite{fu2020d4rl}, which are widely used for developing and evaluating reinforcement learning algorithms. We ran the Hopper and HalfCheetah agents in a straight channel. Their observations consisted of the agent's joint angle information. These agents can receive sparse rewards by reaching the specified locations. We set the length of the channel to 10 units for SparseHopper and 90 units for SparseHalfCheetah. These robots were asked to move as far forward as possible. In the AntMaze-Umaze task, the ant robot was required to walk through the U-shaped maze, as shown in Fig.~\ref{fig:ant_maze}. Its observations consisted of the agent's joint angle information and task-specific attributes. This maze is a benchmarking test environment for maze navigation problems. In the PointMaze tasks, a 2-DoF ball is force-actuated in the Cartesian directions x and y to reach a target goal in a closed maze.
}

\section{Additional experimental results}
\subsection{MMD distance changing trend}
We investigate the changing trend of the MMD distance between the agent's policy $\pi$ and the preferred trajectories in $\mathcal{P}$. Fig.~\ref{fig:grid_mmd} shows the changing trend of the MMD distance between $\pi$ and $\mathcal{P}$ during the training process of LOPE in the grid-world task. In the early stage of the training process, a considerable MMD distance was maintained between $\pi$ and $\pi_b$. Then, the MMD distance dropped rapidly. Fig.~\ref{fig:ant_mmd} shows the changing trend of the MMD distance between $\pi$ and $\pi_b$ during the training process of LOPE in the AntMaze task. The MMD distance decreases gradually during training. Intuitively, the policy $\pi_b$ implied by the preference dataset $\mathcal{P}$ is superior to the current policy $\pi$. Updating the policy parameters towards $\pi_b$ guarantees a higher performance improvement bound. 
\begin{figure}[htb]
  \centering
  \subfloat[]{
    \label{fig:grid_mmd}
    \includegraphics[width=4.2cm]{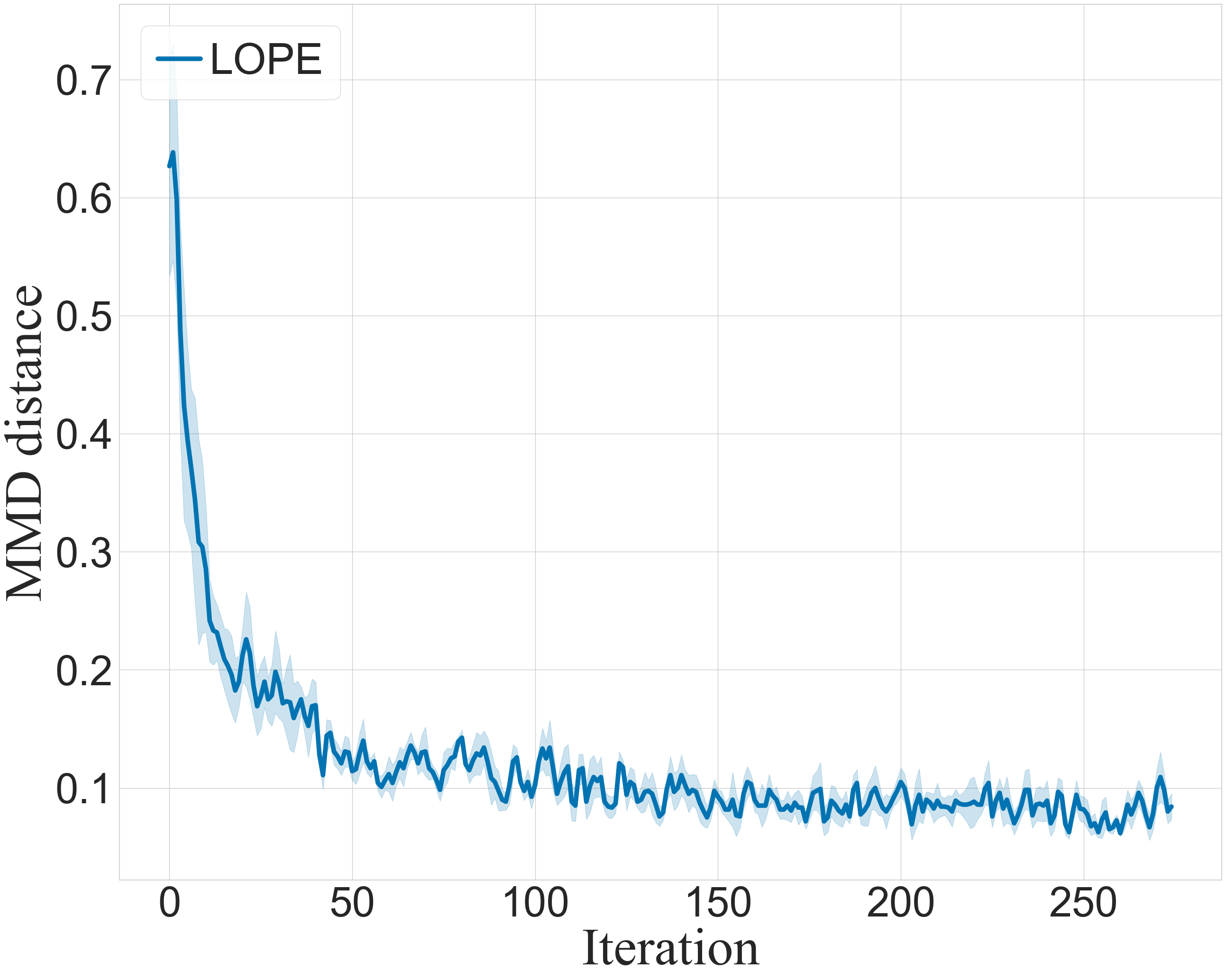}
    }
  \subfloat[]{
    \label{fig:ant_mmd}
    \includegraphics[width=4.2cm]{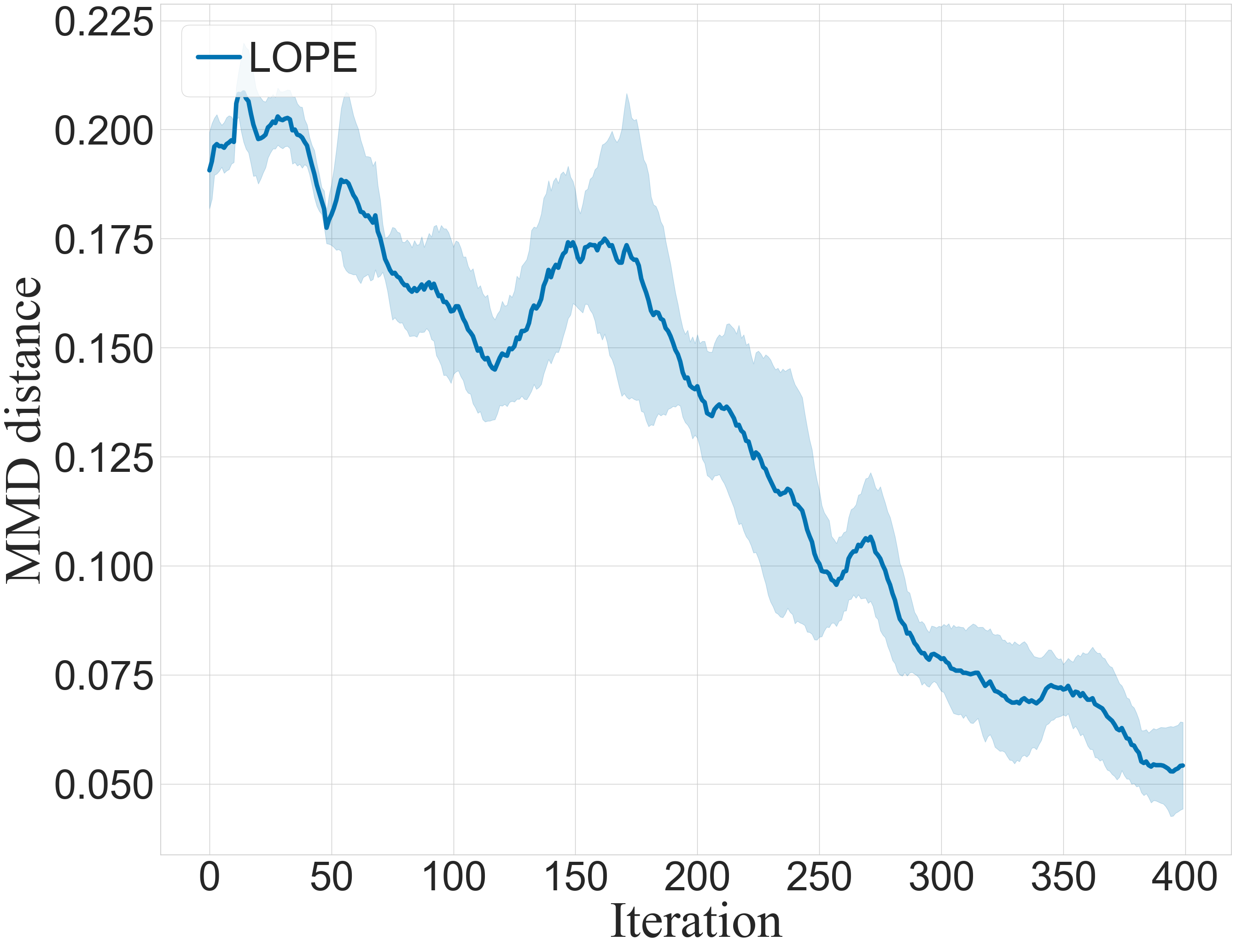}
    }
  \caption{Changing trends of MMD distance in the fixed grid-world maze and ant maze.}
\end{figure}

\bibliographystyle{IEEEtran}
\bibliography{reference}

\vfill

\end{document}